%% file: camera_ready.tex
\documentclass[11pt]{article}

\usepackage[preprint]{acl}

\usepackage{times}
\usepackage{latexsym}
\usepackage{comment}
\usepackage[T1]{fontenc}
\usepackage{makecell}
\usepackage[utf8]{inputenc}
\usepackage{amsmath, amssymb}
\usepackage{amsfonts}
\usepackage{mathrsfs}
\usepackage{dsfont}
\usepackage{multirow}
\usepackage{amsthm}
\newtheorem{definition}{Definition}
\newtheorem{theorem}{Theorem}

\usepackage{tabularx}
\usepackage{booktabs}

\newtheorem{lemma}{Lemma}

\newcommand{\best}[1]{\textbf{#1}}   % smallest-RVR (or peak) highlight
\newcommand{\na}{\textendash}        % pending / not-run cell
\newcommand{\ci}[1]{{\scriptsize$\,\pm#1$}}
\providecommand{\stk}[2]{\shortstack{#1 \\ {\scriptsize$\pm#2$}}}
\usepackage[table]{xcolor}
\usepackage{xcolor}
\definecolor{colREU}{HTML}{8FC4DF}      % soft blue
\definecolor{colDEU}{HTML}{D6A77E}      % muted amber
\definecolor{colRVR}{HTML}{CC7E9A}      % dusty pink
\definecolor{colAK}{HTML}{A998DD}       % soft lavender
\definecolor{colPCTRVR}{HTML}{6F8FE8}   % softened royal blue

\usepackage[T1]{fontenc}            % typewriter \ { } glyphs
\usepackage{listings}               % needed for \begin{lstlisting}
\usepackage{array}                  % needed for p{...} columns in tab:decoding
\usepackage{hyperref}               % must be loaded BEFORE cleveref
\usepackage[capitalize]{cleveref}   % needed for \cref / \Cref

\usepackage{mathtools}
\usepackage{microtype}

\usepackage{inconsolata}
\usepackage{enumitem}

\usepackage{graphicx}
\usepackage[textsize=tiny]{todonotes}
\usepackage{tcolorbox}
\tcbuselibrary{breakable,skins}
\usepackage{xcolor}

\newtcolorbox{prompt}{
    enhanced,
    breakable,
    colback=gray!10,
    colframe=gray!50,
    boxrule=0.5pt,
    arc=2pt,
    left=5pt, right=5pt, top=5pt, bottom=5pt,
    before skip=0pt, after skip=0pt,
    fontupper=\small\ttfamily,
    before upper={\sloppy\emergencystretch=3em\hbadness=10000\relax},
}

\newcommand{\fh}[1]{{\color{purple}{[FH: #1]}}}
\title{In LLM Reasoning, there is Irrationality on top of Value Misalignment}

\author{Kejiang Qian \\
  University of Edinburgh \\
  \texttt{K.Qian-8@sms.ed.ac.uk} 
  \And
  Fengxiang He\\
  University of Edinburgh \\
  \texttt{fhe@ed.ac.uk} \\}

\begin{document}
\maketitle

\begin{abstract}
Significant progress has been made in aligning LLMs with target value functions. We argue that, even when an LLM has been well aligned in (post-)training, it may still fail to maximise the aligned value in reasoning. We mathematically formalise this gap as rational value risk: the utility discrepancy between a model’s deployed reasoning strategy and its rational counterpart whose responses maximise utility in the steepest direction. The estimation error of rational value risk is further decomposed into three components from bounded prompts, bounded responses, and imperfect verifiers. Extensive experiments are conducted, covering models Llama-3.1, Qwen-2.5, T\"ulu-3 families (7B-72B), GPT-5.2, GPT-5.5, and DeepSeek-V4, and benchmarks UltraFeedback, AlpacaEval, GSM8K, MATH, HumanEval, and MathArena. The results validate that (1) rational value risk is widespread; (2) value alignment can reduce, but cannot avoid, it; (3) self-consistency can improve rationality; and (4) a longer chain of thought improves rationality but with diminishing returns. The code is at \url{https://github.com/EVIEHub/LLM-Rationality}.
%These findings suggest that rationality is a distinct dimension of LLM evaluation and optimisation, complementary to value alignment.
%The experiments confirm (1) the existence of rational value risk; (2) its reduction (but not elimination) in value alignment; (3) the sensitivity to reasoning strategy; (4) benefits (but diminishing) of longer reasoning.

\end{abstract}

\input{sections/Intro}

\input{sections/Lit}
\input{sections/Pre}

\input{sections/Rationality}
\input{sections/Theory}
\input{sections/Exp}

\section{Conclusions}

This work identifies \emph{rational value risk} as an inference-time failure that can persist after value alignment: the utility gap between a model's deployed reasoning strategy and its rational counterpart. We decompose the estimation error of this risk into bounded prompts, bounded responses, and imperfect verifiers, and evaluate it across open-weight and proprietary LLMs on conversational, mathematical, and code-generation benchmarks. Our results show that rational value risk is widespread, reduced but not eliminated by post-training, and can be partly mitigated by self-consistency and a longer chain of thought, suggesting rationality as a distinct evaluation dimension complementary to value alignment.

\input{sections/Lim}

%\section*{Acknowledgments}


%\bibliography{custom}

\bibliography{ref}

\appendix

\input{appendices/Proofs}
\input{appendices/Experiment}

\end{document}

%% file: sections/Intro.tex
\section{Introduction}

Significant progress has been made in aligning large language models (LLMs) with target value functions. Through {value alignment} methods such as supervised fine-tuning (SFT) {\citep{ouyang2022training}}, reinforcement learning from human feedback (RLHF) {\citep{christiano2017,stiennon2020,ouyang2022training}}, {direct preference optimisation (DPO) \citep{rafailov2023direct}}, constitutional training {\citep{bai2022constitutional}}, and reinforcement learning with verifiable rewards (RLVR) {\citep{shao2024deepseekmath}}, modern LLMs have become increasingly capable of producing outputs that better reflect human preferences, task-specific objectives, and safety constraints. These advances have substantially improved the helpfulness, reliability, and task performance of LLMs across open-ended dialogue {\citep{thoppilan2022}}, mathematical reasoning {\citep{shao2024deepseekmath}}, code generation {\citep{chen2021evaluating}}, amongst others.

This paper argues that value alignment alone does not guarantee ideal reasoning at inference:

\begin{tcolorbox}[colback=white!98!black,colframe=white!30!black,boxsep=1.1pt,top=6pt,fontupper=\itshape]
\centering
\it
Even if an LLM has been well aligned in training, it may still fail to maximise the aligned value in reasoning. 
\end{tcolorbox}
% \begin{quote}
%     \centering
%     \it
% Even if an LLM has been well aligned in training, it may still fail to maximise the aligned value in reasoning. 
% \end{quote}
%In many reasoning tasks, an LLM may be capable of generating a high-utility answer among its possible responses, but the deployed reasoning strategy may nevertheless produce a lower-utility one. This failure is not simply a matter of the model lacking the right value function. Rather, it reflects a distinct inference-time limitation: the model does not always act rationally with respect to the value it has already learned.

%We argue that value misalignment does not fully account for failures in LLM reasoning. Even when post training successfully induces an aligned value function, deployment-time inference may still fail to produce the response that best maximises the aligned value. 

We mathematically formalise this phenomenon by \emph{rational value risk}, defined as the utility discrepancy between a deployed reasoning strategy and its rational counterpart. Here, we define \emph{rational reasoning} as the strategy whose responses maximise utility in the steepest direction under the given value function. This definition separates two sources of failure that are often conflated in LLM evaluation: (1) value misalignment: the learned value function itself is misaligned, and (2) irrational reasoning: the model fails to produce the responses that maximise the aligned value. %In this sense, rational value risk measures the unrealised utility left on the table by imperfect inference-time reasoning.

%We define a \emph{rational reasoning} as responses that maximise expected utility in the steepest direction. In reasoning tasks, a model may generate a high-utility response among its candidate answers, yet ultimately deploy a lower-utility answer. On top of value alignment error, we refer to this additional source of error as a \emph{rational value risk}: the discrepancy of imperfect reasoning away from the rational counterpart. This distinction echoes recent reinforcement-learning theory \citep{qian2026rationality}, which separates the problem of learning or representing value from the problem of acting rationally with respect to that value. 
{
The rational value risk is not directly accessible. 
% We define a compute-bounded estimator of the expected utility for a reasoning strategy based on a finite number of sampled responses and verifier evaluations, leading to an estimator of rational value risk. 
We define a compute-bounded estimator of expected utility using finite response samples and verifier evaluations, and use it to estimate rational value risk. %that can be applied across both stochastic preference-based tasks and deterministic verifiable reasoning tasks. 
The estimation error of this estimated rational value risk is decomposed into three components: (1) a {prompt sampling error}, which captures the statistical error from evaluating on finitely many prompts, (2) a reasoning sampling error, due to bounded sampling of the deployed and rational reasoning strategies, and (3) {verification error}, arising when the evaluation signal is uncertain or imperfect.
}

We conduct extensive experiments across both open-ended conversational tasks and verifiable reasoning tasks. The evaluation covers the Llama-3.1~\citep{grattafiori2024llama}, Qwen-2.5~\citep{yang2024qwen2}, and T\"ulu-3 model~\citep{lambert2024tulu3} families from 7B to 72B parameters, as well as proprietary models including GPT-5.2~\cite{openai2025gpt52}, GPT-5.5~\cite{openai2026gpt55}, DeepSeek-V4-Pro~\cite{deepseek2026v4}, and DeepSeek-V4-Flash~\cite{deepseek2026v4}. The benchmarks include UltraFeedback~\citep{cui2024ultrafeedback} and AlpacaEval~\citep{dubois2024alpacaeval} for conversational preference evaluation, GSM8K~\citep{cobbe2021gsm8k} and MATH~\citep{hendrycks2021measuring} for mathematical reasoning, HumanEval~\citep{chen2021evaluating} for code generation, and MathArena~\citep{balunovic2025matharena} for challenging deployment-style mathematical reasoning. %The code is at \url{xx}.

The empirical results validate four hypotheses:
\begin{enumerate}[label=\textbf{H\arabic*:}, leftmargin=*]%, itemsep=0pt, topsep=0pt, parsep=0pt, partopsep=0pt]
\item Rational value risk is widespread. Across {eight models of different scales and six benchmarks}, {LLM reasoning consistently deviates from the rational counterpart, resulting in utility loss.}
\item {Value alignment methods} can reduce, but cannot avoid, rational value risk. Rationality is not fully solved by value alignment alone.
{
\item Rationality benefits from consistency across reasoning paths, and is insensitive to sampling diversity (or temperature). %, which has only a limited effect. 
}
\item A longer chain of thought can improve rationality, but its benefits diminish beyond a certain reasoning budget.
\end{enumerate}
%\textbf{H1: rational value risk is widespread.} Across models and benchmarks, LLMs constantly generate high-utility candidates but fail to consistently deploy them. \textbf{H2: post-training reduces, but does not eliminate, rational value risk.} Rationality is not fully solved by value alignment alone. Third, rational value risk is highly sensitive to the inference-time reasoning strategy, including sampling temperature and self-consistency. Fourth, longer reasoning can improve rationality, but its benefits diminish beyond a certain reasoning budget. 

Together, these findings are the first in the literature to %suggest that rationality should be treated as a distinct dimension of LLM evaluation and optimisation, complementary to value alignment., this is the first work to 
formally define and empirically measure rational value risk in LLM reasoning, separating inference-time irrationality from value misalignment, to the best of our knowledge.

%% file: sections/Lit.tex
\section{Related works}

\begin{comment}

\paragraph{LLM reasoning.}
%A parallel line of work improves reasoning performance by modifying
%inference-time search over the model distribution.
Chain-of-thought prompting \citep{wei2022chain} and self-consistency
\citep{wang2023selfconsistency} demonstrate that sampling multiple
reasoning paths can substantially improve reasoning accuracy.
Verifier-guided search and process reward models further improve candidate
selection \citep{cobbe2021gsm8k,uesato2022solving,lightman2023verify},
while recent work studies inference-time compute as a scaling dimension
\citep{snell2024scaling,brown2024largelanguagemonkeys}. Our work is
complementary to these approaches: rather than proposing a new reasoning
strategy, we evaluate existing inference policies through the residual
rational value risk they leave unresolved.

\paragraph{LLM value alignment.}
Modern LLMs are aligned primarily through post-training methods that shape
the response distribution toward human or verifier-defined preferences.
RLHF \citep{ouyang2022training}, Constitutional AI
\citep{bai2022constitutional}, and direct preference optimisation (DPO)
\citep{rafailov2023direct} establish the standard alignment paradigm,
while recent systems further incorporate reinforcement learning with
verifiable rewards (RLVR). These methods improve expected utility under a
reward or preference model, whereas our work studies whether aligned models
can rationally realise high-utility responses at inference time.

\end{comment}

\paragraph{Rationality in machine learning.}
%Rationality in machine learning has been studied from both philosophical and decision-theoretic perspectives.
%\citet{valiant1995rationality} characterise rationality as the ability to utilise information for prediction and control under PAC-style guarantees, while \citet{abel2019concepts} formalise bounded rationality in reinforcement learning through representation complexity and predictive accuracy. 
%The rationality of machine learning has been studied from several perspectives. 
\citet{valiant1995rationality} offers a philosophical account of rationality %defining rationality as the ability to abstract and use available information to understand, predict, and control the environment 
under a PAC-style criterion. 
{\citet{abel2019concepts} defines bounded rationality in reinforcement learning and also shows that rational decision-making involves a trade-off between representational simplicity and predictive accuracy.} From behavioural data, \citet{evans2025rational} model bounded rationality through a Wasserstein constraint between the learned policy and a prior. \citet{sunehag15} derive decision-theoretic axioms for rational reinforcement-learning agents, although these axioms exclude many standard algorithms. %, including those based on $\epsilon$-greedy exploration. 

\citet{qian2026rationality} design tractable rationality measures and develop theory for reinforcement learning agents. The results, however, do not apply to LLM reasoning in deployment. This work addresses the gap through a new suite of rationality measures, theory, and extensive experiments, as part of our continuing efforts to understand the rationality of AI agents.

\paragraph{Rationality in LLMs.}
%LLM rationality has been studied through several partially overlapping lenses. 
Cognitive-science evaluations test whether models exhibit human-like bounded rationality, heuristics, content effects, and cognitive biases \citep{binz2023using,hagendorff2023human,macmillan2024irrationality,yax2024studying,lampinen2024language,coda2024cogbench,malberg2025comprehensive,brady2025dual}. Decision-theoretic and economic approaches instead ask whether model choices are consistent with utility maximisation, risk attitudes, revealed preferences, or preference axioms such as transitivity \citep{chen2023emergence,jia2024decision,song2025benchmarking,liu2025aligning}. A third line audits or improves rational behaviour by enforcing belief consistency, logical preference consistency, debiasing, or rational thought prompting \citep{kassner2023language,echterhoff2024cognitive,koo2024benchmarking,gou2024rationality}. Recent surveys and benchmarks consolidate these views into broader notions of rational agents with consistency, grounding, preference orderability, and evidence-aligned decision making \citep{jiang2025towards,zhou2025rationality}.
%These studies show that rationality is not captured by accuracy or alignment alone, but they mainly evaluate behavioural consistency, biases, or choice axioms. In contrast, our work formalises rationality as an inference-time utility gap between a deployed reasoning strategy and its rational counterpart, thereby separating irrational reasoning from value misalignment.

{Although prior work has shown that LLMs exhibit rationality failures such as inconsistency, cognitive biases, and violations of decision-theoretic principles, it has not explicitly characterised these failures as utility loss that persists after value alignment. %Our work characterises this distinction clearly by defining rational value risk: the gap between the deployed reasoning strategy and the rational counterpart under a fixed aligned value function.
In contrast, our work mathematically formalises rationality as an inference-time utility gap between a deployed reasoning strategy and its rational counterpart, thereby separating irrational reasoning from value misalignment.}

% \brytodo{New Literature review needed. The new measures have been significantly departed from best-of-n / test-time-scaling / verifier / self-consistency gap / pass@k / best@k /  oracle@K / pass@1 / avg@k literature.}
{
\paragraph{Evaluation of LLM reasoning.} The evaluation of LLM reasoning commonly relies on sampling-based performance metrics. {Pass@$1$} reports the success rate of a single sampled response, while {pass@$k$} reports the probability of obtaining at least one correct response within \(k\) samples \citep{chen2021evaluating}. Oracle best-of-\(k\) similarly evaluates the best performance among a finite set of sampled candidates under oracle selection \citep{sunkaraneni2026agenticsystemsboostingweak}. These metrics and their variants quantify the probability or coverage of successful responses from a finite set of samples, and typically rely on binary success signals for tasks with deterministic or verifiable outcomes. %However, strong sample-based performance does not imply rational reasoning, as these metrics cannot measure the gap between the model’s deployed reasoning and reasoning that maximises its aligned value.
}

%% file: sections/Pre.tex
\section{Preliminaries}
% Let $\mathcal{X}$ denote the space of input prompts and $\mathcal{V}$ a finite vocabulary. The output space is $\mathcal{Z} = \bigcup_{T\geq 1}\mathcal{V}^{T}$, the set of finite token sequences. An autoregressive language model with parameters $\theta$ generates a reasoning path $z = (z_1,\dots,z_T)\in\mathcal{Z}$ with a length of $T$ by sequentially sampling tokens according to a stochastic policy conditioned on $x\in\mathcal{X}$:
% \begin{equation}
% \pi_\theta(z \mid x) \;=\; \prod_{t=1}^{T} \pi_\theta\!\left(z_t \mid x, z_{<t}\right).
% \end{equation}

% An extraction function g(·) maps the reasoning path to the final answer yˆ = g(tˆ)
% The vanilla inference mode is sampling, $y \sim \pi_\theta(\cdot \mid x)$. Reasoning is an inference procedure that introduces an intermediate latent sequence $z\in\mathcal{Z}$ between input and output, sampling $y$ from the marginal distribution
% \( y \sim \mathbb{E}_{z\sim \pi_\theta(\cdot\mid x)}\left[\pi_\theta(\cdot\mid x,z)\right].
% \) TODO

Let $\mathcal{X}$ denote the space of input prompts and $\mathcal{V}$ a finite vocabulary. The reasoning space is $\mathcal{Z} = \bigcup_{T\geq 1}\mathcal{V}^{T}$, where $T$ denotes the length of a reasoning path $ z= (z_1,\dots,z_T)$. 

In reasoning, a frozen language model $\pi_\theta$ {parameterised by $\theta\in\Theta$}, where $\Theta$ is the parameter space, generates a reasoning path $z \in \mathcal{Z}$ by sequentially sampling tokens according to a conditional probability on $x\in\mathcal{X}$:
%\begin{equation}
\[\pi_\theta(z \mid x) = \prod_{t=1}^{T} \pi_\theta\left(z_t \mid x, z_{<t}\right).\]
Given a reasoning problems $D = \{\mathbf x_i\}_{i=1}^{M}$, where $\mathbf{x}_i=(x_i,y_i^+)$ of $M$ input questions $x_i \in \mathcal{X}$ and preferred (or verifiable) answer $y_i^+ \in \mathcal{Y}$,  we define a reasoning strategy 
\[d_\theta(\cdot|x_i) \triangleq d(\pi_\theta(\cdot|x_i)),
\]
including temperature sampling, self-consistency, and varying context length. For each input question $x_i$, a frozen language model generates a reasoning path $z_i$ or answer $y_i=g(z_i)$ from an extraction function $g:\mathcal Z\rightarrow\mathcal Y$ by reasoning strategy $d_\theta(\cdot|x_i)$. 

Let a $\mathcal{O}$ be the outcome space, and a verifier is modelled as an outcome distribution 
\[P:\mathcal{X}\times\mathcal{Y}\times\mathcal{Y}\rightarrow \Delta(\mathcal{O})\]
evaluates the generated answer and returns an outcome $o_i \sim P(\cdot|\mathbf x_i,y_i)$. This formulation covers preference-based tasks, where outcomes may be stochastic due to variation across annotators, judges, or repeated evaluations.
In verifiable reasoning tasks, such as mathematical reasoning and code generation, the outcome distribution degenerates to a Dirac distribution as 
\[\delta_{f(\mathbf x_i,y_i)}(o_i):P(o_i=f(\mathbf x_i,y_i)|\mathbf x_i,y_i)=1,\] 
where $$f:\mathcal{X}\times\mathcal{Y}\times\mathcal{Y}\rightarrow \mathcal{O}$$ 
is a deterministic verifier. Let $U: \mathcal{O} \to \mathbb{R}$ be a utility function defined on the outcome space, assigning a utility value $U(o_i)\in[0,1]$ to each evaluation outcome $o_i \in \mathcal{O}$.

%% file: sections/Rationality.tex
\section{Rationality of LLM reasoning}

This section defines the rationality measurement for LLM reasoning.

\subsection{Rationality measures}
We first define (perfectly) rational reasoning. %in inference time from a decision-theoretic standpoin.
% \begin{definition}[rational reasoning]
% \label{def:grr}
% A reasoning path \( z^\circ \) with an answer $y^\circ=g(z^\circ)$ is called \textit{perfectly rational}, if it maximises the expected utility function $U: \mathcal{O} \rightarrow \mathbb{R}$ over the outcome distribution $P(\cdot|\mathbf x,y)$ of any reasoning problem $\mathbf{x}$:
% % \begin{equation*}
% % y^\circ \in \arg\max_{y\in\mathcal{Y}}\mathbb{E}_{o \sim P(\cdot|x,y)}[U(o)].
% % \end{equation*}
% \begin{equation*}
% y^\circ \in \arg\max_{y\in\mathcal{Y}}\mathbb{E}_{o\sim P(\cdot|\mathbf{x},y)}U(o).
% \end{equation*}
% \end{definition}

{
\begin{definition}[rational reasoning]
\label{def:grr}
A reasoning path \( z^\circ \) with an answer $y^\circ=g(z^\circ)$ is called \textit{(perfectly) rational}, if and only if its reasoning strategy maximises the expected utility function $U: \mathcal{O} \rightarrow \mathbb{R}$ over the outcome distribution $P(\cdot|\mathbf x,y)$:
% \begin{equation*}
% y^\circ \in \arg\max_{y\in\mathcal{Y}}\mathbb{E}_{o \sim P(\cdot|x,y)}[U(o)].
% \end{equation*}
\begin{align*}
&y^\circ\sim d_{\theta^\circ}(\cdot|x),\\
&\theta^\circ \in \arg\max_{\theta\in\Theta}\mathbb{E}_{\mathbf{x}\sim \rho,y\sim d_\theta,o\sim P(\cdot|\mathbf{x},y)}U(o).
\end{align*}
% \begin{equation*}
% y^\circ\sim d_{\theta^\circ}(\cdot|x);\theta^\circ \in \arg\max_{\theta\in\Theta}\mathbb{E}_{y\sim d_\theta,o\sim P(\cdot|\mathbf{x},y)}U(o).
% \end{equation*}
\end{definition}
}

% \begin{remark}
% In value alignment, %Post-training methods (e.g., RLHF) updating 
% $\pi_\theta$ is updated to align with $U$ in expectation. In contrast, this paper studies the rationality of a frozen language model $\pi_\theta$ at reasoning time through its reasoning strategy $d_\theta$, isolating the effect of the reasoning from that of training.
% \end{remark}

%\begin{remark}
%In this work, we focus on utility defined over final answers, which is common in preference alignment, mathematical reasoning, and code generation.
%\end{remark}

LLM reasoning is not always rational. We define a rational value risk to quantify the discrepancy of expected utility between an LLM reasoning strategy and its rational counterpart. %and that from a reasoning strategy of frozen language model in Definition \ref{def:rvr}.

{
\begin{definition}[rational value risk]
\label{def:rvr}
For any reasoning question $\mathbf{x}=(x,y^+)$ drawn from a distribution $\rho$, let $y^\circ$ denote an answer of rational reasoning under $d^\circ_\theta(\cdot\mid x)$, and $y$ be an answer drawn from a reasoning strategy $d_\theta(\cdot\mid x)$. We define the rational value risk $\mathcal{R}(d_\theta)$ of $d_\theta$ under a utility function $U$ as:
\begin{align*}
\mathbb{E}_{\substack{
\mathbf{x}\sim\rho,\;
y^\circ\sim d_{\theta^\circ}(\cdot\mid\mathbf{x}),\\
o\sim P(\cdot\mid\mathbf{x},y^\circ)
}}
\left[U(o)\right] -
\mathbb{E}_{\substack{
\mathbf{x}\sim\rho,\;
y\sim d_{\theta}(\cdot\mid\mathbf{x}),\\
o\sim P(\cdot\mid\mathbf{x},y)
}}
\left[U(o)\right].
\end{align*}
\end{definition}
}

{The rational value risk is not directly accessible because it is a population term. To estimate it, we first define a {compute-bounded estimated utility} given a compute budget of $K$ samplings. 
}

% \begin{definition}[compute-bounded rational reasoning]
% \label{def:cbrr}
% Let
% \(
% d_\theta
% \)
% denote the reasoning strategy of a frozen language model. For any reasoning problem \(\mathbf{x}\), let 
% \(
% y_1,\dots,y_K
% \overset{iid}{\sim}
% d_\theta(\cdot\mid x),
% \)
% where each \(y_k=g(z_k)\) is the extracted answer from an independently and identically distributed (iid) sampled reasoning path \(z_k\). For each sampled answer \(y_k\), suppose we obtain \(L\) independently and identically
% distributed evaluation outcomes,
% $
% o_{k,1},\dots,o_{k,L}
% \overset{iid}{\sim}
% P(\cdot\mid \mathbf{x},y_k).$ 
%  We define an empirical expected utility of $y_k$ as 
%  \[\widehat{U}_L(\mathbf{x},y_k)=\mathds{1}\left[\frac{1}{L}\sum_{l=1}^L
% o_{k,l}\ge\frac{1}{2}\right].\] 
% A reasoning path \(\hat z_K^\circ\) with extracted answer \(\hat y_K^\circ=g(\hat z_K^\circ)\) is called \emph{compute-bounded rational} if
% \[
% \hat y_K^\circ
% \in
% \arg\max_{1\le k\le K} \widehat{U}_L(\mathbf{x},y_k).
% \]
% \end{definition}

{
\begin{definition}[compute-bounded estimated utility]
\label{def:cbrr}
Let
\(
d_\theta
\)
denote the reasoning strategy of a language model $\pi_\theta$. For any reasoning problem \(\mathbf{x}\), let 
\[
y_1,\dots,y_k
\overset{iid}{\sim}
d_\theta(\cdot\mid x),
\]
where each \(y_k=g(z_k)\) is the extracted answer from an independently and identically distributed (iid) sampled reasoning path \(z_k\). For each sampled answer \(y_k\), suppose we obtain \(L\) independently and identically distributed evaluation outcomes,
\[
o_{k,1},\dots,o_{k,L}
\overset{iid}{\sim}
P(\cdot\mid \mathbf{x},y_k).
\]
Given a compute budget of $K$ samplings, we define a compute-bounded estimated utility of $\{y_i\}^K_{i=1}$ as: 
%  \[\widehat{U}_L(\mathbf{x},y_k)=\mathds{1}\left[\frac{1}{L}\sum_{l=1}^L
% o_{k,l}\ge\frac{1}{2}\right].\] 
 \[\frac{1}{K}\sum_{k=1}^K\widehat{U}_{L}(\mathbf{x},y_k)=\frac{1}{L}\sum_{k=1}^K\sum_{l=1}^L
U(o_{k,l}).\] 
\end{definition}
}

A Monte Carlo estimator is defined below to estimate the rational value risk. %$\mathcal{R}(d_\theta)$.

% \begin{definition}[empirical rational value risk]
% \label{def:empirical-rvr}
% Given a set of reasoning problem $\{\mathbf x_i\}_{i=1}^M$, let$
% \{\hat y_{i,k}\}_{k=1}^K
% \overset{iid}{\sim}
% d_\theta(\cdot\mid x_i)
% $
% denote $K$ sampled answers for each input $x_i$. Let
% $
% \widehat{U}_L(\mathbf x_i,\hat y_{i,k})
% $
% be the extracted answer of compute-bounded rational reasoning among the $K$ sampled candidates in Definition~\ref{def:cbrr}. The empirical rational value risk $\widehat{\mathcal{R}}_{M,K,L}(d_\theta)$ under a utility function $U$ is defined as follows.
% \begin{align*}
% \label{eq:erg}
% \frac{1}{MK}\sum_{i=1}^{M}\sum_{k=1}^K \left[ \widehat{U}_L(\mathbf x_i,\hat{y}^\circ_{i,K}) - \widehat{U}_L(\mathbf x_i,y_{i,k})\right],
% \end{align*}
% \end{definition}

{
\begin{definition}[estimated rational value risk]
\label{def:empirical-rvr}
Given a set of reasoning problems $\{\mathbf x_i\}_{i=1}^M$, let
\[
\{y_{i,k}\}_{k=1}^K
\overset{iid}{\sim}
d_\theta(\cdot\mid x_i)
\]
denote $K$ sampled answers for each input $x_i$. Let
\[
\{y_{i,k}^{\circ}\}_{k=1}^K \overset{iid}{\sim} d_{\theta^\circ}(\cdot|x_i)
\]
be the rational reasoning answers, and 
$ 
\widehat{U}_{K,L}
$
be the compute-bounded estimated utility among the $K$ sampled candidates in Definition~\ref{def:cbrr}. The estimated rational value risk $\widehat{\mathcal{R}}_{M,K,L}(d_\theta)$ under a utility function $U$ is defined as
\begin{align*}
\frac{1}{MK}\sum_{i=1}^{M}\sum_{k=1}^K \left[ \widehat{U}_L(\mathbf x_i,{y}^\circ_{i,k}) - \widehat{U}_L(\mathbf x_i,y_{i,k})\right].
\end{align*}
\end{definition}
}

\begin{comment}
\paragraph{$\triangleright$ Does $\widehat{U}_L(\mathbf x_i,\hat{y}^\circ_{i,K})$ reduce to $\text{pass@}\kappa$, when verifiers are deterministic?}

\paragraph{Answer:} 
Let's look at $\text{pass@}\kappa$ first:
$$\text{pass@}\kappa = \frac{1}{M}\sum_{i=1}^M \left[1 - \frac{\binom{N - c_i}{\kappa}}{\binom{N}{\kappa}}\right],$$ 
where $N \geq \kappa$  samples are drawn per problem and $c_i = \sum_{j=1}^N \mathds{1}[y_{i,j} = y_i^+]$ is the number of correct samples \citep{chen2021evaluating}. It estimates the probability of solving a verifiable reasoning problem $\mathbf x_i$ within $k\in[K]$ attempts. 

In contrast, our $\widehat{U}_L(\mathbf x_i,\hat{y}^\circ_{i,K})$ measures the maximum empirical utility within the candidate set of size $ K$, i.e., $\max_{k\in [K]} \widehat{U}_L(\mathbf x_i,y_{i,k})$. 

Further, when the verifiers or utility $U(o_i)$ is not deterministic, the rational reasoning $\hat{y}^\circ_{i,K}$ is defined by the highest utility, not necessarily the exact correct response.
\end{comment}

%% file: sections/Theory.tex
\subsection{Theoretical guarantee on estimation} %for estimating rational value risk}
\label{sec:rvr-estimation}

{This section studies the estimation error in empirically computing the rational value risk. Detailed proofs are given in Appendix~\ref{app:proofs}.

The estimation error can be decomposed into three components: {prompt sampling error, reasoning sampling error, and verification error.}

%We now bound the difference between the rational value risk
%\(\mathcal{R}(d_\theta)\) and its empirical estimator $\widehat{\mathcal{R}}_{M,K,L}(d_\theta)$ in the Lemma \ref{lem:rvr-error-decomposition}, which indicates the sources of estimation error.

\begin{lemma}[estimation error decomposition]
\label{lem:rvr-error-decomposition}
The estimation error decomposes as:
\[
\begin{aligned}
&\mathcal{R}(d_\theta)
-
\widehat{\mathcal{R}}_{M,K,L}(d_\theta)
\\
=&
\underbrace{
\mathcal{R}(d_\theta)-\mathcal{R}_M(d_\theta)
}_{\mathrm{(I)\ \text{prompt sampling error}}}
+
\underbrace{
\mathcal{R}_M(d_\theta)
-
\overline{\mathcal{R}}_{M,K}(d_\theta)
}_{\mathrm{(II)\ \text{reasoning sampling error}}}
\\
&+
\underbrace{
\overline{\mathcal{R}}_{M,K}(d_\theta)
-
\widehat{\mathcal{R}}_{M,K,L}(d_\theta)
}_{\mathrm{(III)\ \text{verification error}}}% estimation}} .
\end{aligned}
\]
\end{lemma}

The three terms are respectively bounded below. %I can be bounded below. %by the following Lemma~\ref{lem:compute-approx}.

\begin{lemma}[prompt sampling error bound]
\label{lem:prompt-concentration}
Given \(M\) reasoning problems
\(\{\mathbf x_i\}_{i=1}^M\)
for any \(\delta\in(0,1)\), with probability at least \(1-\delta\), the prompt sampling error can be bounded by,
\[
\left|
\mathcal{R}(d_\theta)
-
\mathcal{R}_M(d_\theta)
\right|
\le
\sqrt{
\frac{2\log(2/\delta)}{M}
}.
\]
\end{lemma}

\begin{lemma}[reasoning sampling error bound]
\label{lem:reasoning-concentration}
For any \(\delta\in(0,1)\), given the sampled reasoning problems and \(K\) sampled reasoning responses for each problem, with probability at least \(1-\delta\), the reasoning sampling error can be bounded by,

\[
\left|
\mathcal{R}_M(d_\theta)
-
\overline{\mathcal{R}}_{M,K}(d_\theta)
\right|
\le
\sqrt{
\frac{2\log(2/\delta)}{MK}
}.
\]
\end{lemma}

\begin{lemma}[verification error bound]
\label{lem:evaluator-estimation}
Assume that all evaluation outcomes are sampled independently, given the sampled reasoning problems and reasoning answers, with \(L\) verifier evaluations for each sampled response.
Then, for any \(\delta\in(0,1)\), with probability at least \(1-\delta\), the verification error can be bounded by,
\[
\left|
\overline{\mathcal{R}}_{M,K}(d_\theta)
-
\widehat{\mathcal{R}}_{M,K,L}(d_\theta)
\right|
\le
\sqrt{
\frac{2\log(2/\delta)}{MKL}
}.
\]
If the verifier is deterministic, this term is zero.
\end{lemma}

We thus obtain the estimation guarantee.

\begin{theorem}[estimation error bound]
\label{thm:rvr-one-sided}
Assume \(U(o)\in[0,1]\). For any
\(\delta\in(0,1)\), with probability at least \(1-\delta\) over the sampled reasoning problems, reasoning answers, and evaluation outcomes,
\[
\begin{aligned}
&
\left|
\mathcal{R}(d_\theta)
-
\widehat{\mathcal{R}}_{M,K,L}(d_\theta)
\right|
\\
\le{}&
\sqrt{
\frac{2\log(6/\delta)}{M}
}
+
\sqrt{
\frac{2\log(6/\delta)}{MK}
}
+
\sqrt{
\frac{2\log(6/\delta)}{MKL}
}.
\end{aligned}
\]
If the verifier is deterministic, the verification error vanishes.
\end{theorem}

\paragraph{Sample complexity.}
% The estimation error scales as
% \[
% O\left(
% M^{-1/2}
% +
% (MK)^{-1/2}
% +
% (MKL)^{-1/2}
% \right).
% \]
% To make each of the three statistical error terms at most
% \(\epsilon\) with high probability, it suffices, up to logarithmic
% factors, that
% \(
% M=O(\epsilon^{-2}),
% \quad
% MK=O(\epsilon^{-2}),
% \quad
% MKL=O(\epsilon^{-2}).
% \)
% Thus, \(M\) controls coverage of the prompt distribution, \(K\) controls estimation of the reasoning strategies, and \(L\) controls estimation of the outcome utility.

To make the three statistical error terms at most \(\epsilon\) with high probability, it suffices to take
\begin{align*}  
M&=O\!\left(\epsilon^{-2}\log(1/\delta)\right),\\
    K &= O\!\left(\epsilon^{-2}\log(1/\delta)/M\right),\\
    L& = O\!\left(\epsilon^{-2}\log(1/\delta)/MK\right).
\end{align*}    
Thus, estimating rational value risk depends on the numbers of reasoning problems, sampled reasoning responses, and verifier evaluations. For verifiable reasoning tasks with a deterministic verifier, the dependence on verifier samples is eliminated.

% \paragraph{Practical insights.}
% The term \(A_K(d_\theta)\)
% is unavoidable: it measures the utility loss incurred when the finite candidate set fails to contain a high-utility answer. In binary deterministic tasks, \(A_K(d_\theta)\) decays exponentially in \(K\) whenever \(p_x>0\); however, if the model never samples a high-utility answer for some prompts, \(p_x=0\), sampling more candidates cannot remove this error.
% For stochastic verifiers, \(q_{i,k}\) denotes a probability that the verifier assigns a positive preference to candidate \(\hat y_{i,k}\). Majority voting can reduce its variance, but it cannot remove verifier bias.

\paragraph{Practical insights.}
Theorem~\ref{thm:rvr-one-sided} shows that the estimation error decreases as the reasoning problem number \(M\), sampled reasoning response number \(K\), and verifier evaluation number \(L\) increase. Among them, the number of reasoning problems \(M\) is the primary factor, as increasing \(M\) reduces all three error terms simultaneously. Increasing \(K\) further reduces reasoning and verification errors, while increasing \(L\) reduces only the verification error. Consequently, when \(M\) and \(K\) are sufficiently large, verification sampling contributes relatively little to the total estimation error. However, increasing \(L\) only improves estimation of the verifier's expected utility $\mathbb{E}_{o\sim P(\cdot|\mathbf{x},y)}[U(o)]$ and cannot remove bias in the verifier itself. %\fh{We therefore consider both the model itself and an external LLM as verifiers in the following experiments to assess the effect of verifier choice and potential verifier bias.}
}

%% file: sections/Exp.tex
\section{Experiments}

Extensive experiments are conducted to verify four empirically verifiable hypotheses, which well establish that \emph{there is irrationality on top of misalignment in LLM reasoning.} %, which are verified in extensive experiments.

\subsection{Empirically verifiable hypotheses}

{
\paragraph{H1: Rational value risk is widespread.} 
%Across different models and benchmarks, 
LLM reasoning in deployment can deviate from rational reasoning, leading to a utility loss. This utility loss is expected to appear across different models and benchmarks, including both conversational and verifiable reasoning tasks. %, and it is not limited to any specific model family, benchmark, or verifier type.

{
\paragraph{H2: Value alignment can reduce, but cannot avoid, irrationality.}
%Rationality cannot be fully improved by value alignment alone. LLMs can be well aligned with a value function but fail to consistently deploy responses that maximise it.
Value alignment shifts the model distribution toward the target value function, but the deployed reasoning strategy can fail to consistently produce responses that maximise that function. Thus, rational value risk remains after value alignment. %by methods  such as SFT, DPO, and RLVR.

{\paragraph{H3: Rationality benefits from self-consistency, and is insensitive to sampling diversity.}
With the LLM parameters fixed, inference-time reasoning strategies can affect its rationality. We study two strategies with different effects: sampling temperature changes the randomness and diversity of sampled reasoning responses, while self-consistency aggregates multiple responses to determine the final answer.}
}

\paragraph{H4: A longer chain of thought improves rationality with diminishing returns.}
Increasing the reasoning length can help the model reduce rational value risk, but this effect diminishes after a certain reasoning budget.

\subsection{Implementation details}

\paragraph{Setup.} %\fh{be clearer on two tasks} %We first measure the empirical rational value risk on general real-world conversation tasks. 
{We evaluate the rationality of LLM reasoning on two task types: \textbf{(1) conversational tasks}: given a conversation dataset $D=\{\mathbf x_i\}_{i=1}^M$ ($\mathbf x_i =(x_i,y_i^+)$), a stochastic verifier compares the LLM's answer $y$ with the human preferred answer $y^+_{i}$, and then returns an outcome $o_i\in\{\text{win},\text{tie},\text{lose}\}$, %corresponding to win, tie, or lose, respectively, 
where $o_i \sim P(\cdot|\mathbf x_i,y_i)$. We then define the utility of each outcome according to the verifier's preference:}
\[
U(o_i) =
\begin{cases}
1,   & o_i=\text{win} \quad (y_i \succ y_i^+)\\
0.5, &  o_i=\text{tie} \quad (y_i \approx y_i^+)\\
0,   & o_i=\text{lose} \quad (y_i \prec y_i^+)
\end{cases},
\]
where $\succ$, $\prec$, and $\approx$ denote verifier preference, non-preference, and indifference, respectively.

The stochastic verifier can be from the LLM itself, reflecting its subjective preference in $P(\cdot|\mathbf x_i,y_i)$, or from larger-scale models as external verifiers.
{\textbf{(2) verifiable reasoning tasks:} we consider binary utility functions defined by answer correctness:}
%correctness defined as
%\begin{equation*}
\(U(f(\mathbf x_i, y_i)) = \mathds{1}\big[y_i = y^+_i\big]\).
%\end{equation*}
% For a fixed prompt $x_i$, we draw $K$ trajectories $\{y_{i,k}\}_{k=1}^{K}$ i.i.d.\ from $\pi_\theta(\cdot\mid x_i)$. The empirical rational gap is computed by averaging over $N$ prompts $\{x_i\}_{i=1}^{N} \stackrel{\text{iid}}{\sim} \rho$:
% \begin{align}
% \hat{\mathcal{R}}_K(\theta) &= U^\circ_K - \bar{U}_K \nonumber \\
% &=\frac{1}{M}\sum_{i=1}^{M}\max_{k\in[K]} U(x_i, y_{i,k}) \nonumber\\ &\qquad -\frac{1}{MK}\sum_{i=1}^{M}\sum_{k=1}^{K} U(x_i, y_{i,k}).
% \end{align}
% The first term $U^\circ_K$ is a sample-based estimator of the reachable utility $U(x_i, y_\theta^\circ)$ in empirical rational gap of Equation~\eqref{eq:erg}. 

% \begin{corolary}
% \label{cor:rkb}

% \end{corolary}
\begin{table*}[t]
\centering
\footnotesize
\setlength{\tabcolsep}{4pt}
\caption{Rational value risk (RVR) across LLMs on conversational and development benchmarks. For conversational tasks, expected utility by rational reasoning (REU) is estimated using DeepSeek-V4-Pro, whereas for verifiable reasoning tasks, REU is set to $1.000$. Compute budget $K{=}64$; values are reported as mean $\pm$ 95\% bootstrap confidence interval. Bold indicates the smallest RVR per dataset. (Self.: self-as-verifier; Ext.: external verifier, DeepSeek-V4-Flash.)}
\label{tab:h1_coloured}
\resizebox{\textwidth}{!}{%
\begin{tabular}{ll *{4}{cc c}}
\toprule
 &  & \multicolumn{3}{c}{Llama-3.1-8B-Instruct} & \multicolumn{3}{c}{Qwen2.5-7B-Instruct} & \multicolumn{3}{c}{Tülu-3-8B-RLVR} & \multicolumn{3}{c}{Qwen2.5-72B-Instruct} \\
\cmidrule(lr){3-5}\cmidrule(lr){6-8}\cmidrule(lr){9-11}\cmidrule(lr){12-14}
Task & Dataset & REU & DEU & RVR & REU & DEU & RVR & REU & DEU & RVR & REU & DEU & RVR \\
\midrule
\multirow{4}{*}{Conversation} & UltraFeedback (self.) & \cellcolor{colREU!39}0.643 & \cellcolor{colDEU!31}0.524 & \cellcolor{colRVR!7}0.119\ci{0.029} & \cellcolor{colREU!34}0.567 & \cellcolor{colDEU!32}0.533 & \cellcolor{colRVR!2}\best{0.034}\ci{0.020} & \cellcolor{colREU!36}0.600 & \cellcolor{colDEU!32}0.532 & \cellcolor{colRVR!4}0.068\ci{0.024} & \cellcolor{colREU!47}0.782 & \cellcolor{colDEU!43}0.712 & \cellcolor{colRVR!4}0.069\ci{0.034} \\
 & UltraFeedback (ext.) & \cellcolor{colREU!43}0.713 & \cellcolor{colDEU!30}0.492 & \cellcolor{colRVR!13}0.221\ci{0.026} & \cellcolor{colREU!43}0.713 & \cellcolor{colDEU!30}0.504 & \cellcolor{colRVR!13}0.208\ci{0.026} & \cellcolor{colREU!43}0.713 & \cellcolor{colDEU!30}0.497 & \cellcolor{colRVR!13}0.215\ci{0.027} & \cellcolor{colREU!43}0.713 & \cellcolor{colDEU!35}0.590 & \cellcolor{colRVR!7}\best{0.122}\ci{0.030} \\
 & AlpacaEval (self.) & \cellcolor{colREU!49}0.812 & \cellcolor{colDEU!41}0.690 & \cellcolor{colRVR!7}0.122\ci{0.020} & \cellcolor{colREU!48}0.793 & \cellcolor{colDEU!40}0.660 & \cellcolor{colRVR!8}0.133\ci{0.021} & \cellcolor{colREU!46}0.763 & \cellcolor{colDEU!45}0.756 & \cellcolor{colRVR!0}0.007\ci{0.020} & \cellcolor{colREU!58}0.970 & \cellcolor{colDEU!58}0.972 & \cellcolor{colRVR!0}\best{-0.002}\ci{0.014} \\
 & AlpacaEval (ext.) & \cellcolor{colREU!54}0.904 & \cellcolor{colDEU!30}0.502 & \cellcolor{colRVR!24}0.402\ci{0.017} & \cellcolor{colREU!54}0.904 & \cellcolor{colDEU!33}0.543 & \cellcolor{colRVR!22}0.361\ci{0.020} & \cellcolor{colREU!54}0.904 & \cellcolor{colDEU!44}0.738 & \cellcolor{colRVR!10}0.166\ci{0.019} & \cellcolor{colREU!54}0.904 & \cellcolor{colDEU!49}0.824 & \cellcolor{colRVR!5}\best{0.079}\ci{0.017} \\
\midrule
\multirow{4}{*}{\shortstack[l]{Verifiable\\reasoning}} & GSM8K & \cellcolor{colREU!60}1.000 & \cellcolor{colDEU!47}0.780 & \cellcolor{colRVR!13}0.220\ci{0.014} & \cellcolor{colREU!60}1.000 & \cellcolor{colDEU!54}0.906 & \cellcolor{colRVR!6}0.094\ci{0.012} & \cellcolor{colREU!60}1.000 & \cellcolor{colDEU!52}0.861 & \cellcolor{colRVR!8}0.139\ci{0.014} & \cellcolor{colREU!60}1.000 & \cellcolor{colDEU!58}0.958 & \cellcolor{colRVR!2}\best{0.042}\ci{0.009} \\
 & MATH & \cellcolor{colREU!60}1.000 & \cellcolor{colDEU!28}0.470 & \cellcolor{colRVR!32}0.530\ci{0.019} & \cellcolor{colREU!60}1.000 & \cellcolor{colDEU!54}0.898 & \cellcolor{colRVR!6}0.102\ci{0.013} & \cellcolor{colREU!60}1.000 & \cellcolor{colDEU!39}0.658 & \cellcolor{colRVR!21}0.342\ci{0.021} & \cellcolor{colREU!60}1.000 & \cellcolor{colDEU!56}0.938 & \cellcolor{colRVR!4}\best{0.062}\ci{0.011} \\
 & HumanEval & \cellcolor{colREU!60}1.000 & \cellcolor{colDEU!26}0.431 & \cellcolor{colRVR!34}0.569\ci{0.056} & \cellcolor{colREU!60}1.000 & \cellcolor{colDEU!45}0.749 & \cellcolor{colRVR!15}\best{0.251}\ci{0.052} & \cellcolor{colREU!60}1.000 & \cellcolor{colDEU!24}0.398 & \cellcolor{colRVR!36}0.602\ci{0.047} & \cellcolor{colREU!60}1.000 & \cellcolor{colDEU!44}0.730 & \cellcolor{colRVR!16}0.270\ci{0.059} \\
 & MathArena & \cellcolor{colREU!60}1.000 & \cellcolor{colDEU!0}0.003 & \cellcolor{colRVR!60}0.997\ci{0.002} & \cellcolor{colREU!60}1.000 & \cellcolor{colDEU!5}0.087 & \cellcolor{colRVR!55}0.913\ci{0.051} & \cellcolor{colREU!60}1.000 & \cellcolor{colDEU!0}0.008 & \cellcolor{colRVR!60}0.992\ci{0.007} & \cellcolor{colREU!60}1.000 & \cellcolor{colDEU!7}0.115 & \cellcolor{colRVR!53}\best{0.885}\ci{0.063} \\
\bottomrule
\end{tabular}%
}
\end{table*}

\paragraph{Datasets and benchmarks.} %\fh{polish the presentation}
{
For conversational tasks, we measure the rational value risk on two benchmarks: UltraFeedback~\citep{cui2024ultrafeedback}
and AlpacaEval~\citep{dubois2024alpacaeval}.
For verifiable reasoning, we use three widely used benchmarks in LLM development and evaluation: GSM8K~\citep{cobbe2021gsm8k},
MATH~\citep{hendrycks2021measuring}, and HumanEval~\citep{chen2021evaluating}. In addition, we adopt one reasoning benchmark as a deployment dataset:
{MathArena, including AIME 2025 and BRUMO 2025~\citep{balunovic2025matharena}}. 
\textit{Its release dates are later than or close to those of the evaluated models}, while it is more difficult than other benchmarks,  inducing a distribution shift between the development inputs, $\mathbf x \sim \rho$, and (unseen or harder) deployment inputs, $\mathbf x \sim \rho^\dagger$. Detailed descriptions are in Appendix~\ref{app:per-dataset}.
}
\paragraph{Models.}
Experiments {for H1-H4} use open-weight models,
Qwen2.5-7B-Instruct~\citep{yang2024qwen2},
Llama-3.1-8B-Instruct~\citep{grattafiori2024llama},
Llama-3.1-T\"ulu-3-8B-(SFT, DPO, RLVR)~\citep{lambert2024tulu3},
%To study rational value risk at larger scale, we additionally evaluate
Qwen2.5-72B-Instruct~\citep{yang2024qwen2},  and
Llama-3.1-T\"ulu-3-70B-(SFT, DPO, RLVR)~\citep{lambert2024tulu3}, and
%We further test frontier 
proprietary models, GPT-5.2 \cite{openai2025gpt52},
GPT-5.5 \cite{openai2026gpt55}, DeepSeek-V4-Pro~\cite{deepseek2026v4}, and DeepSeek-V4-Flash APIs \cite{deepseek2026v4}. %{references}.
{
In practice, the rational reasoning is not directly accessible. For conversational tasks, we assume that a stronger model is more likely to approximate rational reasoning and use DeepSeek-V4-Pro (1.6T parameters)~\citep{deepseek2026v4}, as a practical proxy. We use DeepSeek-V4-Flash (284B parameters)~\citep{deepseek2026v4} as the external stochastic verifier for all evaluated models.
For verifiable reasoning tasks, we set the expected utility of rational reasoning to $1$, since ground-truth answers exist and rational reasoning achieves maximum utility by producing them.}

% For GSM8K and MATH, we use the answer-extraction pipeline from \texttt{lm-evaluation-harness} together with SymPy-based symbolic equivalence checking, following the standard practice introduced by Hendrycks et al. For HumanEval, correctness is determined by executing the model's completion against the provided unit tests.

% \paragraph{Calibration of the sampling budget $K$.}
% The estimator ${U}^\circ_K$ in Equation~\eqref{eq:erg} is a lower bound on the reachable maximized utility that converges monotonically as $K\to\infty$ in $K$. To select a sampling budget that achieves a sufficiently tight estimate of $\hat{\mathcal{R}}_K(\pi_\theta)$ while remaining computationally tractable, we evaluate Tülu-3-8B on a 200-prompt subset of GSM8K at $K \in \{1, 4, 16, 64, 256, 1024\}$ and track both $\hat{\mathcal{R}}_K$ and its GPU-hour cost. 

\paragraph{Configuration.}
Unless otherwise specified, $K {=} 64$ reasoning paths are sampled per prompt using temperature sampling with $\tau {=} 1.0$, without top-$p$ or top-$k$ truncation. We consider two verifier settings on UltraFeedback and AlpacaEval: (1) self-as-verifier with verifier budget $L{=}5$, where the evaluated model acts as its own verifier, and (2) external verifier using a larger-scale language model, such as DeepSeek-V4-Flash as the verifier with budget $L{=}3$. Estimator calibration and verifier details are presented in Appendix~\ref{app:calibration} and Appendix~\ref{app:evaluator}, respectively.

\begin{table}[t]
\centering
\small
\setlength{\tabcolsep}{6pt}
\caption{Rational value risk on MathArena across all evaluated models, including open-weight 7-8B and 70-72B models and proprietary models. As MathArena is a verifiable reasoning task, REU${=}1$. Values are reported as mean $\pm$ 95\% bootstrap confidence intervals.}
\label{tab:matharena_all}
\begin{tabular}{ll c}
\toprule
Size & Model & Rational value risk \\
\midrule
\multirow{3}{*}{7-8B} & Qwen2.5-7B & 0.913\ci{0.051} \\
 & Tülu-3-8B-RLVR & 0.992\ci{0.007} \\
 & Llama-3.1-8B & 0.997\ci{0.002} \\
\midrule
\multirow{2}{*}{70-72B} & Qwen2.5-72B & 0.885\ci{0.063} \\
 & Tülu-3-70B-RLVR & 0.940\ci{0.042} \\
\midrule
\multirow{3}{*}{APIs} & DeepSeek-V4-Flash & 0.414\ci{0.102} \\
 & GPT-5.2 & 0.510\ci{0.106} \\
 & GPT-5.5 & 0.431\ci{0.107} \\
\bottomrule
\end{tabular}
\end{table}

\paragraph{Reproducibility and budget.} 
Experiments are performed with {vLLM 0.6.3} on six NVIDIA A800 80GB GPUs {and APIs including GPT 5.2, GPT 5.5, DeepSeek-V4-Pro, and DeepSeek-V4-Flash}. The experiments require approximately 76.6 GPU-hours. {API cost is \$583 in total, with GPT 5.2(\$117), GPT 5.5(\$117), DeepSeek-V4-Pro(\$77), and DeepSeek-V4-Flash(\$272).} Additional implementation details are provided in Appendix~\ref{app:config}. The code is at \url{https://github.com/EVIEHub/LLM-Rationality}.%{The code is available at \url{https://anonymous.4open.science/r/rationality-of-LLM-reasoning-78C}.}

\subsection{Experimental results}

{
\paragraph{H1: Rational value risk is widespread.} Table~\ref{tab:h1_coloured} shows the expected utility of rational reasoning (REU), expected utility of deployed reasoning (DEU), and their difference, rational value risk (RVR), across conversational, mathematical reasoning, and code generation benchmarks. We observe positive rational value risk across nearly all evaluated settings. 

On conversational tasks, RVR is generally larger under external verification than under self-verification. For example, on UltraFeedback, externally verified RVR ranges from $0.122$ for Qwen2.5-72B to $0.504$ for Qwen2.5-7B, but self verification has smaller RVR values for all models, including $0.068$ for T\"ulu-3-8B and $0.069$ for Qwen2.5-72B. A similar pattern holds on AlpacaEval, where Qwen2.5-72B exhibits an RVR close to zero ($-0.002$) in the self-as-verifier setting. This pattern suggests a self-verification bias that models may prefer their own responses. 

On verifiable reasoning tasks, RVR increases when models are evaluated on more challenging deployment benchmarks. For instance, Qwen2.5-72B exhibits RVRs of only $0.042$ on GSM8K and $0.062$ on MATH, but it rises markedly to $0.885$ on MathArena. This indicates that model reasoning becomes less rational as task difficulty increases.

\begin{table}[t]
\centering
\small
\setlength{\tabcolsep}{5pt}
\caption{Rational value risk of T\"ulu-3-8B family across SFT, DPO, and RLVR stages. Bold indicates the smallest RVR per dataset. Values are reported as mean $\pm$ 95\% bootstrap confidence interval. (Ext.: external verifier, DeepSeek-V4-Flash).}
\label{tab:h2}
\begin{tabular}{ll ccc}
\toprule
Dataset & & SFT & DPO & RLVR \\
\midrule
\multirow{3}{*}{\shortstack[l]{UltraFeedback\\(external verifier)}}
 & REU & 0.713 & 0.713 & 0.713 \\
 & DEU & 0.037 & 0.493 & 0.497 \\
 & RVR & \stk{0.675}{0.027} & \stk{0.219}{0.026} & \stk{\best{0.215}}{0.027} \\
\midrule
\multirow{3}{*}{\shortstack[l]{AlpacaEval\\(external verifier)}}
 & REU & 0.904 & 0.904 & 0.904 \\
 & DEU & 0.020 & 0.728 & 0.738 \\
 & RVR & \stk{0.884}{0.019} & \stk{0.176}{0.020} & \stk{\best{0.166}}{0.019} \\
\midrule
\multirow{3}{*}{GSM8K}
 & REU & 1.000 & 1.000 & 1.000 \\
 & DEU & 0.588 & 0.858 & 0.861 \\
 & RVR & \stk{0.412}{0.016} & \stk{0.142}{0.013} & \stk{\best{0.139}}{0.014} \\
\midrule
\multirow{3}{*}{MATH}
 & REU & 1.000 & 1.000 & 1.000 \\
 & DEU & 0.276 & 0.639 & 0.658 \\
 & RVR & \stk{0.724}{0.016} & \stk{0.361}{0.021} & \stk{\best{0.342}}{0.021} \\
\midrule
\multirow{3}{*}{HumanEval}
 & REU & 1.000 & 1.000 & 1.000 \\
 & DEU & 0.148 & 0.250 & 0.398 \\
 & RVR & \stk{0.852}{0.021} & \stk{0.750}{0.036} & \stk{\best{0.602}}{0.047} \\
\bottomrule
\end{tabular}
\end{table}

Table~\ref{tab:matharena_all} further compares RVR across models of different scales on MathArena. The evaluated 7-8B models exhibit the highest RVR, ranging from $0.913$ to $0.997$. The 70-72B models show only slight reductions, with RVRs of $0.885$ and $0.940$. Proprietary models achieve the lowest RVR, ranging from $0.414$ to $0.510$, but still exhibit a large utility gap relative to rational reasoning on challenging tasks.

These results confirm that rational value risk is widespread across both conversational and verifiable reasoning tasks. They also reveal three key factors associated with reasoning rationality, including task difficulty, model capability, and, for conversational tasks, the choice of verifier. In this context, rational value risk becomes a central bottleneck in inference-time reasoning.

}

{

{
\paragraph{H2: Value alignment reduces, but cannot avoid, irrationality.}
%We next analyse the effect of value alignment on rational value risk during post-training.
Table~\ref{tab:h2} reports REU, DEU, and RVR of the T\"ulu-3-8B family across SFT, DPO, and RLVR stages on conversational and verifiable reasoning benchmarks.
We observe %Table~\ref{tab:h2} shows 
that post-training generally reduces rational value risk, with the largest improvements occurring from SFT to DPO.
For example, RVR decreases from \(0.675\) to \(0.219\) on UltraFeedback, from \(0.412\) to \(0.142\) on GSM8K, and from \(0.724\) to \(0.361\) on MATH. The additional reduction from DPO to RLVR is modest on these benchmarks, although HumanEval shows a larger decrease, from \(0.750\) to \(0.602\). Despite these improvements, rational value risk remains after RLVR, particularly on MATH and HumanEval.

\begin{table}[t]
\centering
\small
\setlength{\tabcolsep}{5pt}
\caption{Rational value risk along the value-alignment pipeline of the T\"ulu-3 models on MathArena. Values are reported as mean $\pm$ 95\% bootstrap confidence interval.}
\label{tab:h2_matharena}
\begin{tabular}{ll ccc}
\toprule
Model & Stage & Rational value risk \\
\midrule
\multirow{3}{*}{Tülu-3-8B} & SFT & 0.995\ci{0.003} \\
 & DPO & \best{0.989}\ci{0.008} \\
 & RLVR &  0.992\ci{0.007} \\
\midrule
\multirow{3}{*}{Tülu-3-70B} & SFT & 0.976\ci{0.018} \\
 & DPO & \best{0.938}\ci{0.042} \\
 & RLVR & 0.940\ci{0.042} \\
\bottomrule
\end{tabular}
\end{table}

\begin{figure}[t]
    \centering
    \includegraphics[width=\linewidth,height=10cm]{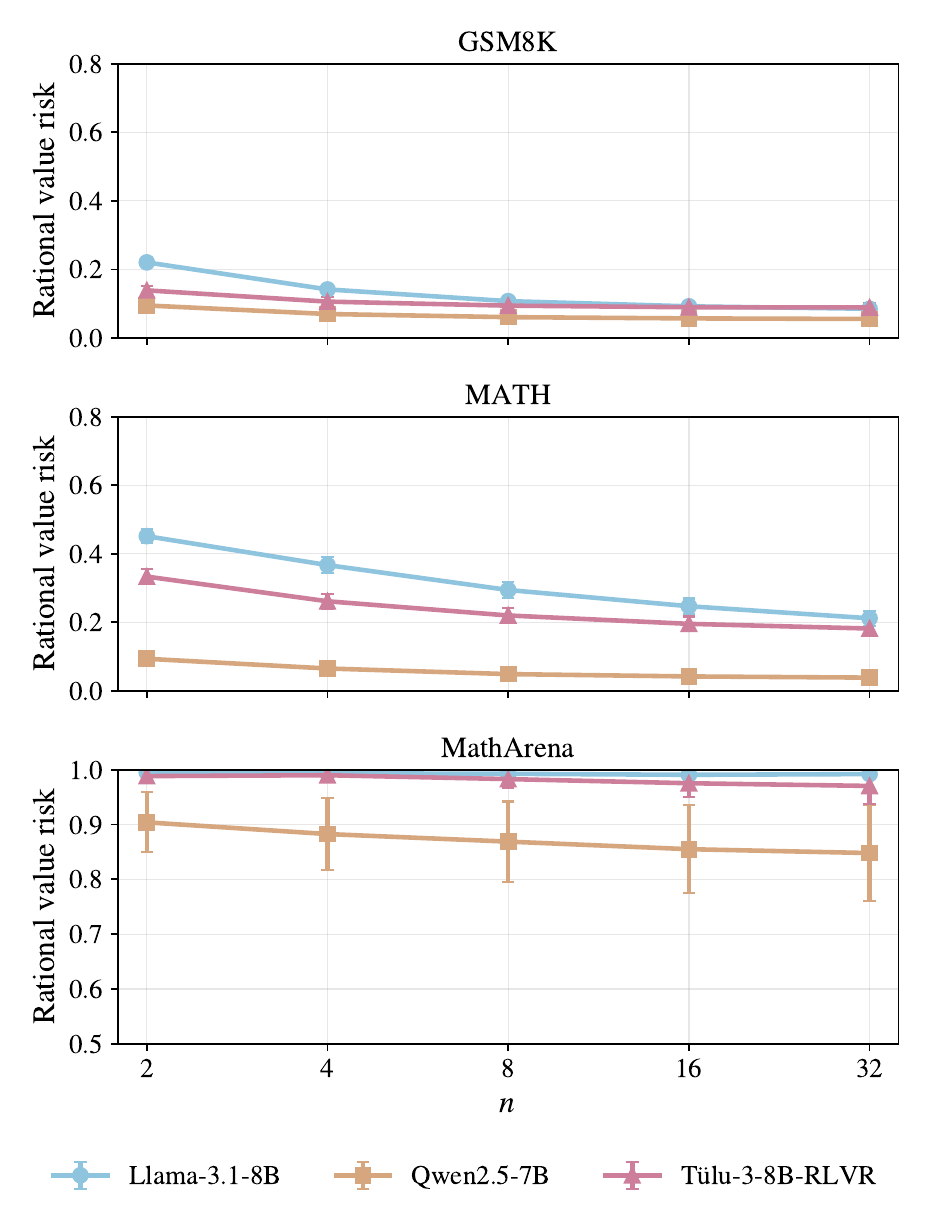}
\caption{Effects of self-consistency under different voting budgets $n=\{2,4,8,16,32\}$ on rational value risk of T\"ulu-3-8B-RLVR, Qwen2.5-7B-Instruct (Qwen2.5-7B), and Llama-3.1-8B-Instruct (Llama-3.1-8B).}
    \label{fig:h3_sc}
\end{figure}

\begin{table}[t]
\centering
\small
\setlength{\tabcolsep}{6pt}
\caption{Effects of sampling temperature $\tau=\{0.0,0.7,1.0\}$ on rational value risk of T\"ulu-3-8B-RLVR (T\"ulu-3-8B), Qwen2.5-7B-Instruct (Qwen2.5-7B), and Llama-3.1-8B-Instruct (Llama-3.1-8B). Values are reported as mean $\pm$ 95\% bootstrap confidence interval.}
\label{tab:h3_temperature}
\begin{tabular}{llccc}
\toprule
Dataset & Model & $\tau{=}0.0$ & $\tau{=}0.7$ & $\tau{=}1.0$ \\
\midrule
\multirow{3}{*}{GSM8K} & Llama-3.1-8B & 0.154 & 0.170 & 0.220 \\
 & Qwen2.5-7B & 0.082 & 0.088 & 0.094 \\
 & Tülu-3-8B & 0.136 & 0.133 & 0.139 \\
\midrule
\multirow{3}{*}{MATH} & Llama-3.1-8B & 0.344 & 0.392 & 0.530 \\
 & Qwen2.5-7B & 0.089 & 0.093 & 0.102 \\
 & Tülu-3-8B & 0.328 & 0.331 & 0.342 \\
\midrule
\multirow{3}{*}{MathArena} & Llama-3.1-8B & 1.000 & 0.989 & 0.997 \\
 & Qwen2.5-7B & 0.900 & 0.906 & 0.913 \\
 & Tülu-3-8B & 1.000 & 0.989 & 0.992 \\
\bottomrule
\end{tabular}
\end{table}

\begin{figure}[t]
    \centering
    \includegraphics[width=\linewidth,height=10.28cm]{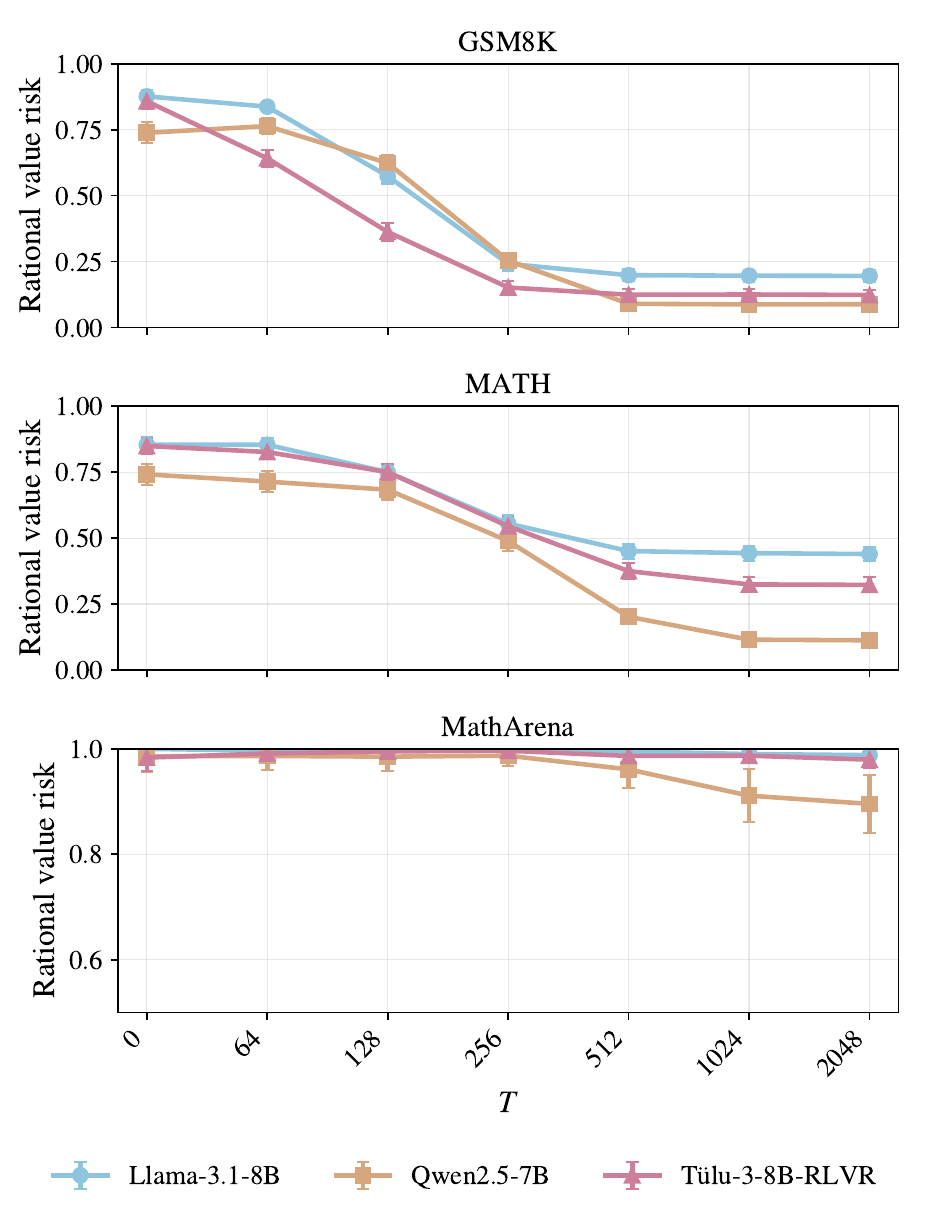}
\caption{Rational value risk of T\"ulu-3-8B-RLVR, Qwen2.5-7B-Instruct (Qwen2.5-7B), and Llama-3.1-8B-Instruct (Llama-3.1-8B) vs. reasoning length $T$.}
    \label{fig:h4_rl}
\end{figure}

Table~\ref{tab:h2_matharena} reports rational value risk for the T\"ulu-3 8B and 70B models across the value-alignment pipeline on the more challenging MathArena benchmark. We observe a different pattern from the previous benchmarks, with value alignment showing limited reductions in RVR. For T\"ulu-3-8B, RVR remains nearly unchanged across SFT, DPO, and RLVR (\(0.995\), \(0.989\), and \(0.992\), respectively). T\"ulu-3-70B shows a larger reduction from \(0.976\) at SFT to \(0.938\) at DPO, but little additional change after RLVR (\(0.940\)). It shows that rationality improvement of T\"ulu-3-8B is not apparent across the alignment stages, suggesting that value alignment is less effective on challenging benchmarks at the 8B scale.

These results indicate that value alignment improves reasoning rationality, particularly through DPO, but rational value risk persists even after RLVR. It is also less effective at reducing rational value risk under increased task difficulty.
}

{
\paragraph{H3: Rationality benefits from self-consistency, but is insensitive to sampling diversity.}
Table~\ref{tab:h3_temperature} reports the effect of sampling temperature (\(\tau\in\{0,0.7,1.0\}\)), while Figure~\ref{fig:h3_sc} reports the effect of self-consistency under different voting budgets (\(n\in\{2,4,8,16,32\}\)).

Table~\ref{tab:h3_temperature} shows that sampling temperature has a small and inconsistent effect on RVR. Qwen2.5-7B and T\"ulu-3-8B remain stable across temperatures on GSM8K and MATH, while Llama-3.1-8B shows slightly larger increases at temperature $\tau=1.0$. On MathArena, RVR stays consistently large across temperature settings. Figure~\ref{fig:h3_sc} illustrates a different pattern for self-consistency. Increasing the voting budget reduces rational value risk, although the reduction becomes less apparent as \(n\) increases. On MATH, RVR decreases from \(0.333\) to \(0.182\) for T\"ulu-3-8B, from \(0.094\) to \(0.039\) for Qwen2.5-7B, and from \(0.451\) to \(0.212\) for Llama-3.1-8B as \(n\) increases from \(2\) to \(32\). Similar trends appear on GSM8K and MATH, but the improvement is weaker on MathArena.

These results indicate rational value risk depends strongly on self-consistency rather than sampling randomness. Temperature sampling improves the diversity of the sampled candidate set, but has a limited effect on rationality. In contrast, self-consistency with more voting budget contributes to rational reasoning. 
}

\paragraph{H4: A longer chain of thought improves rationality with diminishing returns.}
{
Figure~\ref{fig:h4_rl} illustrates the effect of reasoning length \(T \in \{0,64,128,256,512,1024,2048\}\). %on rational value risk.

% Rational value risk decreases monotonically with longer reasoning across all models, with diminishig retuns. For example, on GSM8K of T\"ulu-3-8B-RLVR, it increases from \(0.080\) at \(T=0\) to \(0.456\) at \(T=64\), and then decreases to \(0.107\) at \(T=2048\). The pattern is similar on MATH: its rational value risk increases from \(0.048\) at \(T=0\) to \(0.393\) at \(T=128\), before decreasing when \(T\geq256\). In addition, the benefit of longer reasoning becomes smaller at longer reasoning. On GSM8K and MATH of these three models, their REU and DEU are close to their best values after \(T\geq512\) and \(T\geq1024\), respectively. The MathArena benchmark across three models shows that longer reasoning slightly improves REU, but DEU remains close to zero, while their rational value risk continues to increase.

{
Rational value risk decreases with longer reasoning in most cases, although the rate of reduction decreases at longer chain of thought. On GSM8K, T\"ulu-3-8B-RLVR reduces RVR from \(0.858\) at \(T=0\) to \(0.152\) at \(T=256\), after which further reductions are slight, reaching \(0.123\) at \(T=2048\). Llama-3.1-8B also benefits from a larger reasoning budget, with RVR decreasing from \(0.876\) to \(0.196\). Qwen2.5-7B exhibits a small initial increase from \(0.738\) to \(0.764\), followed by a decrease to \(0.088\). A similar pattern appears on MATH, where for T\"ulu-3-8B-RLVR, RVR decreases steadily from \(0.844\) at \(T=0\) to \(0.311\) at \(T=2048\), with most of the reduction occurring before \(T=512\). These results show that longer reasoning can reduce rational value risk, but the additional benefit diminishes once the reasoning length is sufficiently large. In contrast, a longer chain of thought contributes less to improving rationality across all models on MathArena, where RVR remains high even with extended reasoning.
}

% Longer reasoning generally reduces rational value risk on GSM8K and MATH, with sharply diminishing returns. On GSM8K, RVR falls steeply as the budget grows---for T"ulu-3-8B-RLVR from (0.858) at (T=0) to (0.123) at (T=2048)---and then flattens, changing little once (T\geq512). The same trend holds on MATH but saturates later: T"ulu-3-8B-RLVR drops from (0.844) at (T=0) to (0.311) at (T=2048), with the curve flattening only after (T\geq1024). The decline is not strictly monotonic for every model: Qwen2.5-7B-Instruct on GSM8K first rises slightly from (0.738) at (T=0) to (0.764) at (T=64) before dropping to (0.088). In contrast, on the harder MathArena benchmark longer reasoning barely helps: RVR stays close to (1) across all budgets for all three models (e.g., T"ulu-3-8B-RLVR remains between (0.979) and (0.996)), and only Qwen2.5-7B-Instruct shows a modest reduction, to (0.896) at (T=2048).

These results indicate that an appropriate reasoning-token budget can improve the rationality of LLM reasoning, while avoiding unnecessary token cost.
However, simply extending the reasoning length is insufficient for harder deployment reasoning tasks.}

%% file: sections/Lim.tex
\section*{Limitations}

\paragraph{Rationality over final answers vs. trajectories.} In this work, we define rational value risk over final answers, which is the standard evaluation target in preference alignment, mathematical reasoning, and code generation. This allows rational value risk to be measured using existing outcome-based verifiers, such as preference judges, exact-match verifiers, unit tests, or symbolic checkers. We acknowledge that an alternative formulation could define this measure directly in terms of reasoning paths or intermediate reasoning steps. Process-level rationality evaluates the validity and efficiency of the reasoning trajectory toward the final answer, and may reveal failures that final-answer evaluation does not capture.

However, defining and verifying utility over reasoning processes introduces additional challenges, including comparing multiple reasoning paths, assessing partially correct intermediate steps, and evaluating latent or unfaithful reasoning traces. Therefore, we leave process-level rationality outside the scope of this work and focus on final-answer rationality as a broadly applicable and empirically measurable setting.

\paragraph{New algorithms for improved rationality.} 
Another limitation of this work is that it does not propose a new inference or training algorithm for reducing rational value risk. Instead, our goal is to provide a formal definition, estimator, and empirical diagnosis of irrationality in LLM reasoning. This makes the framework primarily evaluative rather than prescriptive: it identifies when and where a model's reasoning strategy deviates from its rational counterpart, but does not directly provide a method for improvement.

Nevertheless, this work suggests several directions for future applications. Rational value risk can be utilised as an objective for designing inference-time algorithms, such as verifier-guided search, adaptive sampling, self-consistency, or compute allocation strategies that explicitly improve the expected utility of deployed reasoning. It may also inform post-training by distinguishing value misalignment from irrational reasoning under a well-aligned value function. In addition, this measure can be used as a diagnostic tool for comparing models, reasoning strategies, and deployment settings.

{
\paragraph{Human verification.}
Our evaluation of conversational tasks relies on LLM-based verifiers, without human annotations, which may limit the strength of the evaluation evidence. Since conversational tasks do not have ground-truth answers, human verification may provide stronger evidence for assessing verifier reliability. However, large-scale human evaluation is resource-intensive and difficult to scale, and is therefore out of the scope of this work. 

\paragraph{Proprietary models.}
The experiments for H1 include proprietary models, which may limit reproducibility. In contrast, all models used for the experiments on verifiable reasoning tasks in H2-H4 are open-weight models, enabling reproducible evaluation in these settings.
}

%\fh{let's discuss this}

%We note three scope limitations. 
%Our experiments focus on the 7B-14B parameter regime, where measurement is feasible under our compute budget; whether the same patterns persist at the 70B+ scale is left for future work. The Tülu-3 trajectory is a specific instantiation of post-training; conclusions drawn from it should be interpreted as illustrative rather than universally representative of all alignment pipelines. The comparison between Qwen-2.5-7B-Instruct and DeepSeek-R1-Distill-Qwen-7B confounds reasoning training with distillation effects more broadly; we therefore use this comparison as supplementary evidence alongside the controlled inference-procedure experiments on a fixed $\pi_\theta$.

\section*{Ethics considerations}
\label{sec:ethics}

This work is fundamental research. All experiments use publicly available datasets and benchmarks; no human subjects or sensitive data are involved. No direct negative societal impacts are identified.

\section*{Acknowledgements}

KQ was supported in part by the UKRI Grant EP/Y03516X/1 for the UKRI Centre for Doctoral Training in Machine Learning Systems (\href{https://mlsystems.uk/}{https://mlsystems.uk/}).

%% file: appendices/Proofs.tex
\onecolumn
\section{Notation}

\begin{table}[h!]
\centering
\renewcommand{\arraystretch}{1.2}
\caption{Notation}
\begin{tabularx}{0.71\textwidth}{@{} c X @{}}
\toprule
\multicolumn{1}{c}{\textbf{Symbol}} &
\multicolumn{1}{c}{\textbf{Description}} \\
\midrule

$\mathcal{X}$
& Space of input prompts \\

$\mathcal{V}$
& Finite vocabulary \\

$\mathcal{Z}$
& Reasoning space, defined as
$\mathcal{Z}=\bigcup_{T\geq1}\mathcal{V}^{T}$ \\

$T$
& Length of a reasoning path \\

$z$
& Reasoning path, $z=(z_1,\ldots,z_T)$ \\

$\pi_\theta$
& Frozen language model parameterised by $\theta\in\Theta$ \\

$\theta$
& Parameters of the language model \\

$\Theta$
& Space of language model parameters \\

$\mathcal{Y}$
& Answer space \\

$g$
& Extraction function,
$g:\mathcal{Z}\rightarrow\mathcal{Y}$ \\

$D$
& Set of reasoning problems,
$D=\{\mathbf{x}_i\}_{i=1}^{M}$ \\

$\mathbf{x}_i$
& Reasoning problem,
$\mathbf{x}_i=(x_i,y_i^+)$ \\

$x_i$
& Input question \\

$y_i^+$
& Preferred (or verifiable) answer \\

$M$
& Number of reasoning problems \\

$d_\theta$
& Reasoning strategy,
$d_\theta(\cdot\mid x_i)
\triangleq d(\pi_\theta(\cdot\mid x_i))$ \\

$y_i$
& Answer extracted from the generated reasoning path,
$y_i=g(z_i)$ \\

$\mathcal{O}$
& Outcome space \\

$P$
& Outcome distribution of the verifier,
$P:\mathcal{X}\times\mathcal{Y}\times\mathcal{Y}
\rightarrow\Delta(\mathcal{O})$ \\

$o_i$
& Evaluation outcome,
$o_i\sim P(\cdot\mid\mathbf{x}_i,y_i)$ \\

$f$
& Deterministic verifier,
$f:\mathcal{X}\times\mathcal{Y}\times\mathcal{Y}
\rightarrow\mathcal{O}$ \\

$U$
& Utility function defined on the outcome space,
$U:\mathcal{O}\rightarrow\mathbb{R}$ with
$U(o_i)\in[0,1]$ \\

$K$
& Number of sampled reasoning responses for each reasoning problem \\

$L$
& Number of verifier evaluations for each sampled response \\

$\mathcal{R}(d_\theta)$
& Rational value risk of reasoning strategy $d_\theta$ \\

$\mathcal{R}_M(d_\theta)$
& Rational value risk based on $M$ sampled reasoning problems \\

$\overline{\mathcal{R}}_{M,K}(d_\theta)$
& Rational value risk based on $M$ sampled reasoning problems and
$K$ sampled reasoning responses for each problem \\

$\widehat{\mathcal{R}}_{M,K,L}(d_\theta)$
& Estimated rational value risk based on $M$ reasoning problems,
$K$ sampled reasoning responses, and $L$ verifier evaluations \\
\bottomrule
\end{tabularx}
\end{table}

\newpage

\section{Proofs}
\label{app:proofs}

This appendix collects all proofs omitted from the main text.

\subsection{Additional definitions}

For notational brevity, we first define the expected utility of an answer $y$ to a reasoning problem $\mathbf{x}$ as:
\[
V(\mathbf{x},y)
\triangleq
\mathbb{E}_{o\sim P(\cdot\mid\mathbf{x},y)}[U(o)].
\]
Throughout this section, we assume \(U(o)\in[0,1]\), and therefore
\(V(\mathbf{x},y)\in[0,1]\).

For a given prompt \(\mathbf{x}\), we define the expected utility of rational reasoning strategy $d_{\theta^\circ}$,
\[
V^\circ(\mathbf{x})
\triangleq
\mathbb{E}_{y^\circ\sim d_{\theta^\circ}(\cdot\mid x)}
V(\mathbf{x},y^\circ),
\]
and that of the deployed reasoning strategy $d_\theta$,
\[
V_\theta(\mathbf{x})
\triangleq
\mathbb{E}_{y\sim d_\theta(\cdot\mid x)}
V(\mathbf{x},y).
\]
Then, the rational value risk $\mathcal{R}(d_\theta)$ can be written as
\[
\mathcal{R}(d_\theta)
=
\mathbb{E}_{\mathbf{x}\sim\rho}
\left[
V^\circ(\mathbf{x})-V_\theta(\mathbf{x})
\right].
\]

Given \(M\) sampled prompts, we define
\[
\mathcal{R}_M(d_\theta)
=
\frac1M
\sum_{i=1}^M
\left[
V^\circ(\mathbf{x}_i)
-
V_\theta(\mathbf{x}_i)
\right].
\]

Given \(K\) sampled answers from the deployed reasoning strategy and its rational counterpart for every prompt, we define
\[
\overline{\mathcal{R}}_{M,K}(d_\theta)
=
\frac{1}{MK}
\sum_{i=1}^M\sum_{k=1}^K
\left[
V(\mathbf{x}_i,y^\circ_{i,k})
-
V(\mathbf{x}_i,y_{i,k})
\right],
\]
where
\(
y^\circ_{i,k}
\overset{\mathrm{iid}}{\sim}
d_{\theta^\circ}(\cdot\mid x_i),
\quad
y_{i,k}
\overset{\mathrm{iid}}{\sim}
d_\theta(\cdot\mid x_i).
\)

Based on the definition~\ref{def:empirical-rvr}, 
the estimated rational value risk is
\[
\widehat{\mathcal{R}}_{M,K,L}(d_\theta)
=
\frac{1}{MK}
\sum_{i=1}^M\sum_{k=1}^K
\left[
\widehat U_L(\mathbf{x}_i,y^\circ_{i,k})
-
\widehat U_L(\mathbf{x}_i,y_{i,k})
\right].
\]

\subsection{Proof of Lemma~\ref{lem:rvr-error-decomposition}}

\begin{proof}
Adding and subtracting
\(\mathcal{R}_M(d_\theta)\) and
\(\overline{\mathcal{R}}_{M,K}(d_\theta)\) gives
\[
\begin{aligned}
\mathcal{R}(d_\theta)
-
\widehat{\mathcal{R}}_{M,K,L}(d_\theta)
=&
\mathcal{R}(d_\theta)
-
\mathcal{R}_M(d_\theta)
+
\mathcal{R}_M(d_\theta)
-
\overline{\mathcal{R}}_{M,K}(d_\theta)
\\
&+
\overline{\mathcal{R}}_{M,K}(d_\theta)
-
\widehat{\mathcal{R}}_{M,K,L}(d_\theta).
\end{aligned}
\]
This proves the claim.
\end{proof}

\subsection{Proof of Lemma~\ref{lem:prompt-concentration}}

\begin{proof}
For each sampled prompt and its preferred answer $\mathbf{x}_i$, we define
\[
G_i
\triangleq
V^\circ(\mathbf{x}_i)
-
V_\theta(\mathbf{x}_i),
\] which measures the difference between the expected utility of the rational reasoning strategy $d_{\theta^\circ}$ and that of the deployed reasoning strategy $d_\theta$.

Since
$
V^\circ(\mathbf{x}_i),V_\theta(\mathbf{x}_i)\in[0,1],
$
their difference is bounded as
\(
G_i\in[-1,1].
\)
Moreover, \(G_1,\ldots,G_M\) are iid because
\(\mathbf{x}_1,\ldots,\mathbf{x}_M\overset{\mathrm{iid}}{\sim}\rho\).

because the reasoning problems are sampled independently as
$
\mathbf{x}_1,\ldots,\mathbf{x}_M
\overset{\mathrm{iid}}{\sim}\rho,
$
the random variables $G_1,\ldots,G_M$ are independently and identically distributed as well.

By definition, we have
\[
\mathbb{E}[G_i]
=
\mathcal{R}(d_\theta),
\qquad
\mathcal{R}_M(d_\theta)
=
\frac1M\sum_{i=1}^M G_i.
\]
Based on the Hoeffding's inequality \cite{hoeffding1963}, for any $\epsilon \in (0,1)$, their absolute difference can be bounded by
\begin{equation}
\label{eq1}
\Pr\left(
\left|
\mathcal{R}_M(d_\theta)
-
\mathcal{R}(d_\theta)
\right|
\ge \epsilon
\right)
\le
2\exp\left(
-\frac{M\epsilon^2}{2}
\right).
\end{equation}

For any $\delta \in (0,1)$, we set $\delta$ to the right-hand side of equation (\ref{eq1}),
\[
2\exp\left(-\frac{M\epsilon^2}{2}\right) = \delta.
\]
Then, we have
\[
\epsilon
=
\sqrt{\frac{2\log(2/\delta)}{M}}.
\]
Therefore, with probability at least \(1-\delta\),
\[
\left|
\mathcal{R}_M(d_\theta)
-
\mathcal{R}(d_\theta)
\right|
\le
\sqrt{\frac{2\log(2/\delta)}{M}},
\]
which completes the proof.

\end{proof}

\subsection{Proof of Lemma~\ref{lem:reasoning-concentration}}

\begin{proof}
Given the sampled reasoning problems
\(\mathbf{x}_1,\ldots,\mathbf{x}_M\), we first define
\[
Z_{i,k}
\triangleq
V(\mathbf{x}_i,y^\circ_{i,k})
-
V(\mathbf{x}_i,y_{i,k}).
\]
Since \(V(\mathbf{x},y)\in[0,1]\), their difference can be bounded as 
\(
Z_{i,k}\in[-1,1].
\)

Then, the reasoning answers from rational reasoning strategy $d_{\theta^\circ}$ and deployed reasoning strategy $d_\theta$ are independent and we have
\begin{align*}
\mathbb{E}[Z_{i,k}\mid\mathbf{x}_i]
&=
\mathbb{E}_{y^\circ\sim d_{\theta^\circ}(\cdot\mid x_i)}
V(\mathbf{x}_i,y^\circ)-
\mathbb{E}_{y\sim d_\theta(\cdot\mid x_i)}
V(\mathbf{x}_i,y)
=
V^\circ(\mathbf{x}_i)-V_\theta(\mathbf{x}_i).
\end{align*}

Consequently,
\[
\frac1{MK}
\sum_{i=1}^M\sum_{k=1}^K
\mathbb{E}[Z_{i,k}\mid\mathbf{x}_i]
=
\mathcal{R}_M(d_\theta),
\]
while
\[
\frac1{MK}
\sum_{i=1}^M\sum_{k=1}^K
Z_{i,k}
=
\overline{\mathcal{R}}_{M,K}(d_\theta).
\]

Applying Hoeffding's inequality \cite{hoeffding1963} conditionally on the $M$ reasoning problems $\mathbf{x}_{1:M}$, we obtains
\begin{equation}
\label{eq2}
\Pr\left(
\left|
\overline{\mathcal{R}}_{M,K}(d_\theta)
-
\mathcal{R}_M(d_\theta)
\right|
\ge\epsilon
\middle|
\mathbf{x}_{1:M}
\right)
\le
2\exp\left(
-\frac{MK\epsilon^2}{2}
\right).
\end{equation}
Setting the right-hand side of equation (\ref{eq2}) to \(\delta\), we have
\[
\epsilon
=
\sqrt{
\frac{2\log(2/\delta)}{MK}
}.
\]
With probability at least $1-\delta$,
\[
\left|
\overline{\mathcal{R}}_{M,K}(d_\theta)
-
\mathcal{R}_M(d_\theta)
\right| \leq \sqrt{
\frac{2\log(2/\delta)}{MK}
}
\]
This proves the result.
\end{proof}

\subsection{Proof of Lemma~\ref{lem:evaluator-estimation}}

\begin{proof}
For each sampled answer $y^\circ_{i,k}$ and $y_{i,k}$, let the outcome from rational reasoning be
$
o^\circ_{i,k,l}
\overset{\mathrm{iid}}{\sim}
P(\cdot\mid\mathbf{x}_i,y^\circ_{i,k})
$
and outcome from deployed reasoning be
$
o_{i,k,l}
\overset{\mathrm{iid}}{\sim}
P(\cdot\mid\mathbf{x}_i,y_{i,k}).
$
We define
\[
W_{i,k,l}
\triangleq
U(o^\circ_{i,k,l})
-
U(o_{i,k,l}).
\]
Since \(U(o)\in[0,1]\), it can be bounded as $
W_{i,k,l}\in[-1,1].
$

Given the sampled reasoning problems $\mathbf{x}_i$, we define the difference of expected utility between the rational and deployed reasoning answers, $y^\circ_{i,k}$, $y_{i,k}$, as
\[
\begin{aligned}
\mathbb{E}
[
W_{i,k,l}
\mid
\mathbf{x}_i,y^\circ_{i,k},y_{i,k}
]
&=
V(\mathbf{x}_i,y^\circ_{i,k})
-
V(\mathbf{x}_i,y_{i,k}).
\end{aligned}
\]

Therefore, we have
\[
\widehat{\mathcal{R}}_{M,K,L}(d_\theta)
=
\frac{1}{MKL}
\sum_{i=1}^M
\sum_{k=1}^K
\sum_{l=1}^L
W_{i,k,l},
\]
and
\[
\overline{\mathcal{R}}_{M,K}(d_\theta)
=
\frac{1}{MKL}
\sum_{i=1}^M
\sum_{k=1}^K
\sum_{l=1}^L
\mathbb{E}
[
W_{i,k,l}
\mid
\mathbf{x}_i,y^\circ_{i,k},y_{i,k}
].
\]

Using Hoeffding's inequality \cite{hoeffding1963} conditionally on the reasoning problems $\mathbf{x}_i$ and reasoning answers, $y^\circ_{i,k}$ and $y_{i,k}$, we obtain
\begin{equation}
\label{eq3}
\Pr\left(
\left|
\widehat{\mathcal{R}}_{M,K,L}(d_\theta)
-
\overline{\mathcal{R}}_{M,K}(d_\theta)
\right|
\ge\epsilon
\right)
\le
2\exp\left(
-\frac{MKL\epsilon^2}{2}
\right).
\end{equation}

Taking
$
\epsilon
=
\sqrt{
2\log(2/\delta)/MKL
},
$
the upper bound on the left-hand side of equation (\ref{eq3}) is \(\delta\). Therefore, with probability at least \(1-\delta\),
\[
\left|
\widehat{\mathcal{R}}_{M,K,L}(d_\theta)
-
\overline{\mathcal{R}}_{M,K}(d_\theta)
\right|
\le
\sqrt{
\frac{2\log(2/\delta)}{MKL}
}.
\]

If the verifier is deterministic, for every reasoning problem $\mathbf{x}$ and its deployed reasoning answer $y$, we have
$
\widehat U_L(\mathbf{x},y)
=
V(\mathbf{x},y),
$
and hence
\[
\widehat{\mathcal{R}}_{M,K,L}(d_\theta)
=
\overline{\mathcal{R}}_{M,K}(d_\theta).
\]
So the verification error is zero, which completes the proof.
\end{proof}

\subsection{Proof of Theorem~\ref{thm:rvr-one-sided}}

\begin{proof}
By Lemma~\ref{lem:rvr-error-decomposition} and the triangle inequality,
\[
\begin{aligned}
\left|
\mathcal{R}(d_\theta)
-
\widehat{\mathcal{R}}_{M,K,L}(d_\theta)
\right|
\le&
\left|
\mathcal{R}(d_\theta)
-
\mathcal{R}_M(d_\theta)
\right|
+
\left|
\mathcal{R}_M(d_\theta)
-
\overline{\mathcal{R}}_{M,K}(d_\theta)
\right|
\\
&+
\left|
\overline{\mathcal{R}}_{M,K}(d_\theta)
-
\widehat{\mathcal{R}}_{M,K,L}(d_\theta)
\right|.
\end{aligned}
\]

Apply Lemma~\ref{lem:prompt-concentration},
Lemma~\ref{lem:reasoning-concentration}, and
Lemma~\ref{lem:evaluator-estimation}, each with failure probability \(\delta/3\). By a union bound, these three events hold simultaneously with probability at least $1-\delta$. Combining the three bounds gives
\[
%\begin{aligned}
%&
\left|
\mathcal{R}(d_\theta)
-
\widehat{\mathcal{R}}_{M,K,L}(d_\theta)
\right|
\le{}
\sqrt{
\frac{2\log(6/\delta)}{M}
}
+
\sqrt{
\frac{2\log(6/\delta)}{MK}
}
+
\sqrt{
\frac{2\log(6/\delta)}{MKL}
}.
%\end{aligned}
\]
This proves the result.

If the verifier is deterministic,
Lemma~\ref{lem:evaluator-estimation} gives zero verification sampling
error, and the third term vanishes.
\end{proof}

\newpage

%% file: appendices/Experiment.tex
% =====================================================================
% Appendix
% =====================================================================

\newpage
\twocolumn

% =====================================================================
\section{Additional implementation details}
\label{app:config}

This appendix provides additional details on the experiment configurations and implementations.

\subsection{Sampling configuration}
\label{app:config:decoding}

All experiments use a common sampling interface based on vLLM.
Unless otherwise specified, we sample \(K=64\) candidate answers per prompt with
\[
(\texttt{temperature}, \texttt{top\_p}, \texttt{top\_k})
=
(1.0, 1.0, -1),
\]
where \(\texttt{top\_k}=-1\) disables top-\(k\) truncation.
This default configuration is used for H1, H2, and H4.
H3 varies the inference-time reasoning strategy by changing the sampling temperature and the self-consistency budget.
For greedy decoding, we use \(\tau=0\) with \(K=1\), since multiple greedy samples would be identical.

\Cref{tab:decoding} summarises the sampling configuration for each experiment.

\begin{table*}[t]
\centering
\small
\setlength{\tabcolsep}{4pt}
\caption{Sampling hyperparameters used in the experiments.}
\label{tab:decoding}
\begin{tabular}{lccccc p{6.8cm}}
\toprule
Experiment & \(\tau\) & top-\(p\) & top-\(k\) & \(K\) & max\_tokens & Notes \\
\midrule
H1 & 1.0 & 1.0 & \(-1\) & 64 & 1024 & 7-8B models use max\_tokens \(=512\) on AlpacaEval and UltraFeedback datasets.\\
H2 & 1.0 & 1.0 & \(-1\) & 64 & 1024 & 7-8B models use max\_tokens \(=512\) on AlpacaEval and UltraFeedback datasets. Same decoding setting across SFT, DPO, and RLVR stages.\\
H3 direct, \(\tau=0\) & 0.0 & 1.0 & \(-1\) & 1 & 1024 & Greedy decoding. \\
H3 direct, \(\tau=0.7\) & 0.7 & 1.0 & \(-1\) & 64 & 1024 & Stochastic sampling. \\
H3 direct, \(\tau=1.0\) & 1.0 & 1.0 & \(-1\) & 64 & 1024 & Default stochastic sampling. \\
H3 self-consistency & 1.0 & 1.0 & \(-1\) & 64 & 1024 & Voting budget \(n\in\{2,4,8,16,32\}\). \\
H4 & 1.0 & 1.0 & \(-1\) & 64 & \(T\) & \(T\in\{0,64,128,256,512,1024,2048\}\). \\
Self-as-verifier calls & 0.7 & 1.0 & \(-1\) & 1 & 8 & \(L=5\) verifier calls per candidate. \\
External-verifier calls & 0.7 & 1.0 & \(-1\) & 1 & 256 & \(L=3\) verifier calls per candidate; DeepSeek-v4-flash. \\
\bottomrule
\end{tabular}
\end{table*}

\begin{table*}[t]
\centering
\small
\caption{Self-as-verifier diagnostics on the UltraFeedback of H1 experiment. A-pick rate is the fraction of non-tie verifier calls in which the response in position A is selected. Krippendorff's \(\alpha\) measures consistency among the \(L=5\) repeated verifier calls for the same candidate answer. It is computed as \(1-D_o/D_e\), where \(D_o\) is the observed disagreement among verifier calls and \(D_e\) is the average disagreement under random verifier evaluations. A larger value indicates more consistent verifier evaluations.
A value close to \(0.5\) indicates small position bias.}
\label{tab:appendix_position_bias}
\begin{tabular}{lccc}
\toprule
Verifier & A-pick rate & Krippendorff's \(\alpha\) & Mean \(|\)margin\(|\) \\
\midrule
T\"ulu-3-8B-RLVR & 0.502 & 0.599 & 3.63 \\
Qwen2.5-7B-Instruct & 0.501 & 0.595 & 3.67 \\
Llama-3.1-8B-Instruct & 0.502 & 0.639 & 3.82 \\
Qwen2.5-72B-Instruct & 0.441 & 0.732 & 3.59 \\
\bottomrule
\end{tabular}
\end{table*}

\begin{table*}[t]
\centering
\small
\setlength{\tabcolsep}{4pt}
\caption{Extraction functions and their conditional-correctness rates of the GSM8K}
\label{tab:extract-fail}
\begin{tabular}{l cc cc cc cc cc r}
\toprule
\multirow{2}{*}{Model} & \multicolumn{2}{c}{\texttt{\#\#\#\# \(N\)}} & \multicolumn{2}{c}{\texttt{\textbackslash boxed\{\(N\)\}}} & \multicolumn{2}{c}{``the answer is \(N\)''} & \multicolumn{2}{c}{Last number} & \multicolumn{2}{c}{No number} & \multirow{2}{*}{\(M{\cdot}K\)} \\
\cmidrule(lr){2-3} \cmidrule(lr){4-5} \cmidrule(lr){6-7} \cmidrule(lr){8-9} \cmidrule(lr){10-11}
 & Fires & Acc. & Fires & Acc. & Fires & Acc. & Fires & Acc. & Fires & Acc. & \\
\midrule
T\"ulu-3-8B-RLVR        & 99.9 & 86.2 & \(<0.1\) & 62.5 & \textemdash & \textemdash & \(<0.1\) & 25.0 & \textemdash & \textemdash & 84416 \\
Qwen2.5-7B-Instruct   & 84.9 & 90.6 & 14.7   & 90.9 & \textemdash & \textemdash & 0.4   & 72.1 & \textemdash & \textemdash & 84416 \\
Llama-3.1-8B-Instruct &  1.6 & 74.9 & \(<0.1\) & 78.3 & 22.0 & 83.4 & 76.3  & 76.5 & \(<0.1\) & \textemdash & 84416 \\
\bottomrule
\end{tabular}
\end{table*}

\paragraph{Stop tokens.}
We do not configure additional stop tokens.
For instruct models, the corresponding chat template already includes the assistant-end token, and generation stops when the EOS token is produced.
For the T\"ulu-3 base model used in the few-shot H2 setting, no chat template is applied, and generation is controlled by \texttt{max\_tokens}.

\subsection{Prompting setup}
\label{app:config:prompts}

\paragraph{Chat-mode prompts.}
For each chat-mode model and dataset pair, we apply the model's HuggingFace chat template to a fixed system--user message pair.
The system prompt is dataset-specific and specifies the required answer format.
The user prompt contains the input question or problem.
This design keeps the prompting protocol consistent across models while allowing task-specific answer-format instructions.

~

\begin{prompt}
GSM8K system:
  Solve the following math problem step by step. After your reasoning,
  write the final numeric answer on its own line in the format:
  \#\#\#\# <answer>\\

MATH system:
  Solve the following math problem. Show your reasoning, then put your
  final answer in \boxed{}.\\

HumanEval system:
  Complete the following Python function. Return ONLY the function
  definition; do not include explanations or examples.\\

MathArena system:
  Solve the following competition math problem. Show your reasoning
  step by step, then put your final answer inside \boxed{}.\\

UltraFeedback system:  (empty string)
AlpacaEval   system:  (empty string)
  -- in both, each generator/verifier applies its own default system prompt.\\

User templates:\\
  GSM8K          : "{question}"\\
  MATH           : "{problem}"\\
  MathArena      : "{problem}"\\
  HumanEval      : "{prompt}"\\
  UltraFeedback  : "{prompt}"\\
  AlpacaEval     : "{instruction}"\\
\end{prompt}

\paragraph{Few-shot prompts for the T\"ulu-3 base model.}
The T\"ulu-3 base model does not use a chat template.
In the H2 base-stage experiments, we prepare a fixed few-shot context for the user question.
For GSM8K, we use the standard 8-shot chain-of-thought prompt from \citet{wei2022chain}.
For maths, we use a 4-shot algebra prompt that matches the algebra-only split used in our experiments.

\begin{prompt}
GSM8K 8-shot block (header lines only; full text in repo):

  Question: There are 15 trees in the grove. ...\\
  Answer:   ... The answer is 6.\\
  
  Question: If there are 3 cars in the parking lot ...\\
  Answer:   ... The answer is 5.\\
  (... 6 more shots ...)\\
  
  Question: Olivia has \$23. She bought five bagels for \$3 each. ...\\
  Answer:   ... The answer is 8.\\

MATH 4-shot block (algebra-only, header lines only):

  Problem:  Find the sum of all integer values of n such that 20/(2n-1)
            is an integer.\\
  Solution: ... The sum of these values is 1 + 0 + 3 + (-2) = \boxed{2}.\\
  
  Problem:  How many positive 3-digit integers have all digits different?\\
  Solution: ... 9 * 9 * 8 = \boxed{648}.\\
  
  Problem:  A right cylinder ... lateral surface area?\\
  Solution: ... 2 * pi * 2 * 5 = \boxed{20\pi}.\\
  
  Problem:  Compute sqrt((31)(30)(29)(28)+1).\\
  Solution: ... = \boxed{869}.
\end{prompt}

\paragraph{Verifier prompts.}
For UltraFeedback, the verifier compares the generated response with the reference response and returns a win, a tie, or a lose.
We use \(L\) verifier evaluations per candidate and return the averaged outcomes. These responses and preferred answer $y^+$ positions are randomised across calls to reduce position bias. Invalid evaluations are mapped to ties.

\begin{prompt}
System: You are evaluating two responses to a user's question. Pick
the better one. Answer with exactly one character: A if Response A is
better, B if Response B is better, T if they are equally good.\\

User:   User question:
        {x}

        --- Response A ---
        {a}

        --- Response B ---
        {b}\\
        
        Which is better? Answer with one letter (A, B, or T):
\end{prompt}

\subsection{Benchmark details}
\label{app:evaluator}

% \subsection{Verifiable reasoning tasks}

\paragraph{GSM8K.}
We extract the final numeric answer using a priority order over common answer formats, including the canonical GSM8K answer marker, boxed answers, explicit ``the answer is'' statements, and the last numeric expression in the output.
The predicted and reference answers are normalised before numerical comparison.

\paragraph{MATH and MathArena.}
We extract the predicted answer from the final expression \texttt{\textbackslash boxed\{\}} and evaluate it using symbolic equivalence checking. In H1-H3, MATH answers that fail symbolic parsing are marked incorrect. To avoid confounding reasoning-length comparisons with parsing failures or missing boxed answers in H4. For MATH in H4, we prompt the model with
\texttt{\textbackslash n\textbackslash nFinal answer: \textbackslash boxed\{}
to encourage it to directly commit a final answer in
\texttt{\textbackslash boxed\{\}}.
We then extract the final boxed answer and use a SymPy-based symbolic verifier to check whether it is mathematically equivalent to the reference answer. Outputs without a valid boxed answer are treated as unparseable. The MathArena benchmark in H1-H4 follows the same protocol.

\paragraph{HumanEval.}
Code generation tasks are evaluated by executing the generated solution against the provided unit tests in an isolated subprocess with time limits.
A candidate receives utility \(1\) only if all tests pass, and utility \(0\) otherwise.

% \subsection{Conversational tasks}

\paragraph{UltraFeedback.} We use a self-as-verifier and an external verifier to compare each generated response with the reference answer $y^+$, respectively.
The verifier returns win, tie, or lose, which is mapped to utilities \(1\), \(0.5\), and \(0\).
We use multiple verifier evaluations and average their outcomes.
The positions of generated and reference answers are randomised to reduce position bias.
\Cref{tab:appendix_position_bias} reports diagnostics for the self-as-verifier setting and shows that the A/B position bias is small.

\begin{table*}[t!]
\centering
\small
\setlength{\tabcolsep}{4pt}
\caption{MATH verify failure rates, reported as percentages of candidate answers.}
\label{tab:math-failures}
\begin{tabular}{l ccccccc r}
\toprule
Model & No \texttt{\textbackslash boxed\{\}} & Empty GT & Parse error & Parse empty & Verify exception & Incorrect & Correct & \(M{\cdot}K\) \\
\midrule
T\"ulu-3-8B-RLVR        &  4.9 & \textemdash & \textemdash & \(<0.1\) & \textemdash & 29.3 & 65.8 & 64000 \\
Qwen2.5-7B-Instruct   &  2.6 & \textemdash & \textemdash & \(<0.1\) & \textemdash &  7.7 & 89.8 & 64000 \\
Llama-3.1-8B-Instruct & 27.7 & \textemdash & \textemdash & \(<0.1\) & \textemdash & 25.2 & 47.0 & 64000 \\
\bottomrule
\end{tabular}
\end{table*}

\begin{table}[t]
\centering
\small
\caption{GPU hours for each experiment.}
\label{tab:appendix_gpu_hours}
\begin{tabular}{lcc}
\toprule
Hypothesis & Number of cells & GPU-hours \\
\midrule
H1 & 34 & 22.5 \\
H2 & 20 & 7.7 \\
H3 & 54 & 20.0 \\
H4 & 68 & 26.4 \\
\midrule
\textbf{Total} & 176 & 76.6 \\
\bottomrule
\end{tabular}
\end{table}

\begin{table*}[t]
\centering
\small
\setlength{\tabcolsep}{4pt}
\caption{Fraction of unparseable answers across models on GSM8K, MATH, and MathArena with the increased reasoning lengths $T$ in H4 experiment.}
\label{tab:appendix_D5_unparseable_by_T}
\begin{tabular}{ll rrrrrrr}
\toprule
Dataset & Model & $T{=}0$ & $T{=}64$ & $T{=}128$ & $T{=}256$ & $T{=}512$ & $T{=}1024$ & $T{=}2048$ \\
\midrule
GSM8K & T\"ulu-3-8B-RLVR & 0.00 & 0.00 & 0.00 & 0.00 & 0.00 & 0.00 & 0.00 \\
 & Qwen2.5-7B-Instruct & 0.00 & 0.00 & 0.00 & 0.00 & 0.00 & 0.00 & 0.00 \\
 & Llama-3.1-8B-Instruct & 0.18 & 0.02 & 0.00 & 0.00 & 0.00 & 0.00 & 0.00 \\
MATH & T\"ulu-3-8B-RLVR & 0.02 & 0.06 & 0.07 & 0.09 & 0.06 & 0.01 & 0.01 \\
 & Qwen2.5-7B-Instruct & 0.00 & 0.06 & 0.02 & 0.02 & 0.03 & 0.00 & 0.00 \\
 & Llama-3.1-8B-Instruct & 2.34 & 3.32 & 3.00 & 2.00 & 1.28 & 0.98 & 0.91 \\
MathArena & T\"ulu-3-8B-RLVR & 0.00 & 0.03 & 42.68 & 59.14 & 48.72 & 10.68 & 0.49 \\
 & Qwen2.5-7B-Instruct & 0.00 & 0.00 & 0.00 & 0.00 & 0.03 & 0.00 & 0.00 \\
 & Llama-3.1-8B-Instruct & 0.00 & 1.48 & 1.88 & 1.51 & 1.30 & 0.94 & 0.55 \\
\bottomrule
\end{tabular}
\end{table*}

\begin{table*}[t]
\centering
\small
\setlength{\tabcolsep}{4pt}
\caption{Effect of compute budget \(K\) on rational value risk. Each entry is the prompt mean at budget \(K\). The last column of each risk row includes the 95\% prompt-bootstrap confidence interval.}
\label{tab:appendix_saturation}
\begin{tabular}{ll rrrrrrr}
\toprule
Dataset & Model & $K{=}1$ & $K{=}2$ & $K{=}4$ & $K{=}8$ & $K{=}16$ & $K{=}32$ & $K{=}64$ \\
\midrule
GSM8K & Tülu-3-8B-RLVR & 0.142 & 0.139 & 0.139 & 0.139 & 0.137 & 0.139 & 0.139 \\
 & Qwen2.5-7B-Instruct & 0.093 & 0.096 & 0.093 & 0.093 & 0.094 & 0.093 & 0.094 \\
 & Llama-3.1-8B-Instruct & 0.219 & 0.223 & 0.218 & 0.218 & 0.219 & 0.220 & 0.220 \\
MATH & Tülu-3-8B-RLVR & 0.318 & 0.331 & 0.336 & 0.338 & 0.339 & 0.340 & 0.342 \\
 & Qwen2.5-7B-Instruct & 0.096 & 0.099 & 0.097 & 0.100 & 0.102 & 0.103 & 0.102 \\
 & Llama-3.1-8B-Instruct & 0.542 & 0.526 & 0.528 & 0.528 & 0.528 & 0.531 & 0.530 \\
HumanEval & Tülu-3-8B-RLVR & 0.537 & 0.555 & 0.602 & 0.645 & 0.638 & 0.619 & 0.602 \\
 & Qwen2.5-7B-Instruct & 0.238 & 0.238 & 0.239 & 0.242 & 0.250 & 0.245 & 0.251 \\
 & Llama-3.1-8B-Instruct & 0.524 & 0.552 & 0.546 & 0.557 & 0.563 & 0.569 & 0.569 \\
MathArena & Tülu-3-8B-RLVR & 1.000 & 1.000 & 0.992 & 0.992 & 0.994 & 0.992 & 0.992 \\
 & Qwen2.5-7B-Instruct & 0.900 & 0.900 & 0.912 & 0.915 & 0.914 & 0.915 & 0.913 \\
 & Llama-3.1-8B-Instruct & 1.000 & 1.000 & 1.000 & 0.998 & 0.998 & 0.997 & 0.997 \\
\bottomrule
\end{tabular}
\end{table*}

\subsection{Answer extraction analysis}

We further analyse answer-extraction failures for GSM8K and MATH.
For GSM8K, \Cref{tab:extract-fail} reports the used extraction function $g(\cdot)$ and the conditional correctness of each function.
The residual unparseable rate is below \(0.1\%\), indicating that the extractor covers almost all generated numeric answers.

For MATH, \Cref{tab:math-failures} shows that most failures come from missing final boxed answers, rather than symbolic parser errors.
This means that the main source of \(U=0\) is the model's failure to provide the required final-answer format or a correct answer.

{
In the H4 experiment, we specify answer formats and extraction rules for each dataset. For GSM8K, we prompt the model with
\texttt{\textbackslash n\textbackslash nFinal answer:}
and extract the answer using the following formats in order:
\begin{itemize}
    \item \texttt{\#\#\#\# N};
    \item \texttt{\textbackslash boxed\{N\}};
    \item \texttt{the answer is N};
    \item the last standalone number.
\end{itemize}

% Table \ref{tab:appendix_D5_unparseable_by_T} provides detailed unparseable outputs on GSM8K and MathArena per reasoning budget $T$. The rate are nealy zero. 
% Except for T\"ulu-3-8B-RLVR on T\"ulu-3-8B which has high unparseable rate from 128 to 512, based on the figure

Table \ref{tab:appendix_D5_unparseable_by_T} reports the rate of unparseable outputs on GSM8K, MATH and MathArena across reasoning lengths \(T\). The rates are nearly zero in most settings. One exception is T\"ulu-3-8B-RLVR on MathArena, which exhibits a relatively high unparseable rate for \(T=128\) to \(T=512\). However, this does not affect the conclusion in Figure~\ref{fig:h4_rl}, as rational value risk remains high on MathArena even at \(T=2048\).
}

\subsection{Implementations}
\label{app:config:stack}

Experiments are run with Python 3.11, vLLM 0.6.3, PyTorch 2.5.1, CUDA 12.4, and Transformers 4.45.2.
The main open-weight model experiments are run on six NVIDIA A800 GPUs with 80GB of memory.
All reported confidence intervals are 95\% prompt-bootstrap intervals with \(B=1000\) resamples.
We use a fixed seed for each headline cell and cache generations by model, dataset, sampling configuration, prompt template, and sampling budget, so repeated runs reuse the same generated candidates when the configuration is unchanged.

\subsection{Computational cost}
\label{app:config:cost}

The sampling stage requires 76.6 GPU-hours in total.
Verification and bootstrap aggregation are performed separately and are not included in the GPU-hour count.
\Cref{tab:appendix_gpu_hours} reports the GPU hours for each experiment.

% =====================================================================
\section{Effect of sample sizes on estimation error}
\label{app:calibration}

We study the estimated rational value risk along the three terms in Theorem~\ref{thm:rvr-one-sided}: the number of reasoning problems \(M\), the number of sampled reasoning responses \(K\), and the number of verifier verifications \(L\).

\subsection{Effect of reasoning problem size}

We assess prompt sampling error by bootstrapping subsets of prompts of size \(M\).
\Cref{tab:appendix_M_convergence} shows that the confidence interval decreases as \(M\) grows, which is consistent with the \(1/\sqrt{M}\) term in Theorem~\ref{thm:rvr-one-sided}. \Cref{tab:appendix_B1_dataset_sizes} reports the number of reasoning problems $M$ for each experiment and dataset.

\begin{table}[t]
\centering
\small
\caption{Bootstrap CI width under different prompt sampling numbers \(M\).}
\label{tab:appendix_M_convergence}
\begin{tabular}{rccc}
\toprule
$M$ & Tülu-3-8B & Qwen2.5-7B & Llama-3.1-8B \\
\midrule
50 & 0.197$\pm$0.082 & 0.088$\pm$0.060 & 0.266$\pm$0.082 \\
100 & 0.151$\pm$0.058 & 0.092$\pm$0.043 & 0.223$\pm$0.057 \\
200 & 0.143$\pm$0.039 & 0.092$\pm$0.028 & 0.213$\pm$0.038 \\
500 & 0.125$\pm$0.022 & 0.087$\pm$0.018 & 0.210$\pm$0.023 \\
1319 & 0.139$\pm$0.014 & 0.094$\pm$0.012 & 0.220$\pm$0.013 \\
\bottomrule
\end{tabular}
\end{table}
\label{app:calibration:M}

\begin{table}[t]
\centering
\small
\setlength{\tabcolsep}{6pt}
\caption{The number of reasoning problem $M$ per experiment and dataset. Every cell draws $K{=}64$ samples per prompt; `-' marks a dataset not run in that experiment.}
\label{tab:appendix_B1_dataset_sizes}
\begin{tabular}{l rrrr}
\toprule
Dataset & H1 & H2 & H3 & H4 \\
\midrule
AlpacaEval & 300 & 300 & 300 & - \\
UltraFeedback & 300 & 300 & 500 & - \\
\midrule
GSM8K & 1319 & 1319 & 1319 & 500 \\
MATH & 1000 & 1000 & 1000 & 500 \\
MathArena & 60 & 60 & 60 & 60 \\
HumanEval & 164 & 164 & 164 & - \\
\bottomrule
\end{tabular}
\end{table}

\begin{table*}[t!]
\centering
\small
\setlength{\tabcolsep}{5pt}
\caption{Rational value risk for each difficulty level at \(K=64\).}
\label{tab:appendix_per_difficulty}
\begin{tabular}{lll ccc}
\toprule
Dataset & Bucket & \(n\) & T\"ulu-3-8B-RLVR & Qwen2.5-7B-Instruct & Llama-3.1-8B-Instruct \\
\midrule
MATH & L1 & 113 & 0.066 & 0.009 & 0.211 \\
 & L2 & 170 & 0.174 & 0.039 & 0.350 \\
 & L3 & 226 & 0.271 & 0.049 & 0.462 \\
 & L4 & 234 & 0.350 & 0.088 & 0.591 \\
 & L5 & 257 & 0.631 & 0.245 & 0.792 \\
\midrule
HumanEval & Short & 56 & 0.571 & 0.165 & 0.462 \\
 & Medium & 53 & 0.557 & 0.250 & 0.575 \\
 & Long & 55 & 0.677 & 0.340 & 0.672 \\
\bottomrule
\end{tabular}
\end{table*}

\begin{table}[t]
\centering
\small
\caption{Effect of \(L\) on rational value risk at \(K=32\) for the H1 experiment: T\"ulu-3-RLVR on UltraFeedback. (Std. is the standard deviation.)}
\label{tab:appendix_L_sensitivity}
\begin{tabular}{cccc}
\toprule
$L$ & DEU & Std. \\
\midrule
1 & 0.558 & 0.0102 \\
2 & 0.559 & 0.0094 \\
3 & 0.560 & 0.0095 \\
4 & 0.559 & 0.0098 \\
5 & 0.559 & 0.0094 \\
\bottomrule
\end{tabular}
\end{table}

\subsection{Effect of reasoning response size}
\label{app:calibration:K}

To assess the effect of the reasoning response size, we compute the estimator at truncated budgets
\(K\in\{1,2,4,8,16,32,64\}\) using the same cached candidate set.
For each \(K\), we compute the estimated rational value risk $\widehat{\mathcal R}_{K,M,L}$ in Definition~\ref{def:empirical-rvr},
\[
\frac{1}{MK}\sum_{i=1}^{M}\sum_{k=1}^{K}\left[\widehat{U}_L(\mathbf{x}_i,y^\circ_{i,k})-\widehat{U}_L(\mathbf{x}_i,y_{i,k})\right].
\]

% \Cref{tab:appendix_saturation} reports the effect of compute budget \(K\) on rational value risk in verifiable reasoning tasks when $L=1$. It is steady from \(K=32\) to \(K=64\) across the reported settings, indicating that \(K=64\) provides a practical compute budget for estimating rational value risk.

\Cref{tab:appendix_saturation} reports the effect of reasoning response size on rational value risk in verifiable reasoning tasks with \(L=1\). The estimates are stable because \(M\) is sufficiently large, with only minor changes as \(K\) increases. In particular, the results remain nearly unchanged from \(K=32\) to \(K=64\) across the reported settings, suggesting that \(K=64\) provides a practical compute budget for estimating rational value risk.

\subsection{Effect of verifier evaluation size}
\label{app:calibration:L}

For deterministic verifiers, such as GSM8K, MATH, HumanEval, and MathArena, the verifier outcome is fixed for a given answer. So repeated verifier verifications are not required, i.e., $L=1$.

In conversational benchmarks, the number of verifier verifications $L$ may affect the estimation error of a stochastic verifier. For UltraFeedback, we average the verifiers' outcomes at different \(L\). 
\Cref{tab:appendix_L_sensitivity} shows that the rational value risk estimation is stable across \(L\), indicating that the verifier evaluation term is not the dominant source of estimation error in this setting, supported by Lemma~\ref{lem:evaluator-estimation}.

% =====================================================================

% =====================================================================
\section{Rational value risk across various task difficulties}
\label{app:per-dataset}
We provide a breakdown of difficulty for the MATH and HumanEval benchmarks to investigate the pattern of rational value risk across varying difficulty levels in Table \Cref{tab:appendix_per_difficulty}.

For MATH, we use the official difficulty levels from 1 to 5.
For HumanEval, which has no official difficulty label, we group problems by the token length of the reference solution into short, medium, and long groups.

The results show that rational value risk appears across difficulty levels. On MATH, the risk generally increases with difficulty level, especially for weaker models.
On HumanEval, the risk also tends to increase with solution length.

% =====================================================================
% \section{Failure-case analysis}
% \label{app:failures}

% We also inspect prompts where the sampled candidate set contains at least one correct answer, but the model has low average utility across samples.
% Formally, we select prompts satisfying
% \[
% \mathrm{REU}(x_i)=1
% \quad\text{and}\quad
% \mathrm{DEU}(x_i)\leq 0.1.
% \]
% These cases illustrate the core phenomenon measured by rational value risk: the model can generate a correct or high-utility answer within \(K\) samples, but the deployed reasoning strategy still concentrates on lower-utility answers.
% \Cref{tab:failure_candidates} reports the number of such prompts in H1 experiment.

% =====================================================================
\section{Statements on AI Assistants in Research and Writing}
\label{app:ai-usage}

{We used LLM assistants for language polishing, theory refinement, and code refactoring during the development of this work. All scientific claims, theoretical results, experimental results, and analyses were produced and verified by the authors; no AI-generated text was included without author review.}

% \begin{table}[t]
% \centering
% \small
% \caption{Number of failure-case prompts satisfying \(\mathrm{REU}=1\) and \(\mathrm{DEU}\leq 0.1\).}
% \label{tab:failure_candidates}
% \begin{tabular}{llr}
% \toprule
% Model & Dataset & Count \\
% \midrule
% T\"ulu-3-8B-RLVR & GSM8K & 35 / 1319 \\
% T\"ulu-3-8B-RLVR & MATH & 68 / 1000 \\
% T\"ulu-3-8B-RLVR & HumanEval & 29 / 164 \\
% Qwen2.5-7B-Instruct & GSM8K & 22 / 1319 \\
% Qwen2.5-7B-Instruct & MATH & 15 / 1000 \\
% Qwen2.5-7B-Instruct & HumanEval & 9 / 164 \\
% Llama-3.1-8B-Instruct & GSM8K & 45 / 1319 \\
% Llama-3.1-8B-Instruct & MATH & 115 / 1000 \\
% Llama-3.1-8B-Instruct & HumanEval & 23 / 164 \\
% Qwen2.5-72B-Instruct & GSM8K & 9 / 1319 \\
% Qwen2.5-72B-Instruct & MATH & 7 / 1000 \\
% Qwen2.5-72B-Instruct & HumanEval & 19 / 164 \\
% \bottomrule
% \end{tabular}
% \end{table}

\hfill